\documentclass[twoside,11pt]{article}

\usepackage{blindtext}

\usepackage[preprint]{jmlr2e}
\usepackage[utf8]{inputenc} % allow utf-8 input
\usepackage[T1]{fontenc}    % use 8-bit T1 fonts
\usepackage{hyperref}       % hyperlinks
\usepackage{url}            % simple URL typesetting
\usepackage{booktabs}       % professional-quality tables
\usepackage{amsfonts}       % blackboard math symbols
\usepackage{nicefrac}       % compact symbols for 1/2, etc.
\usepackage{microtype}      % microtypography
\usepackage{float}
\usepackage{graphicx}
\usepackage[dvipsnames]{xcolor}         % colors
\hypersetup{
    colorlinks,
    linkcolor={blue!50!black},
    citecolor={blue!50!black},
    urlcolor={blue!80!black}
}

\usepackage{amsmath}
\usepackage{amssymb}
\usepackage{mathtools}
\usepackage{amsthm}
\usepackage[capitalize,noabbrev]{cleveref}

\usepackage{nicefrac}
\usepackage{bbm,bm}
\usepackage{enumitem}
\usepackage{thmtools,thm-restate}
\usepackage{parskip}
\usepackage{caption}

\declaretheorem[name=Theorem]{thm}
\crefname{thm}{Theorem}{Theorems}
\declaretheorem[name=Lemma]{lem}
\crefname{lem}{Lemma}{Lemmas}
\declaretheorem[name=Proposition]{prop}
\crefname{prop}{Proposition}{Propositions}
\declaretheorem[name=Corollary]{cor}
\crefname{cor}{Corollary}{Corollaries}

\crefname{assmpt}{Assumption}{Assumptions}
\declaretheorem[name=Fact]{fact}
\crefname{fact}{Fact}{Facts}
\theoremstyle{definition}
\declaretheorem[name=Definition]{defn}
\crefname{defn}{Definition}{Definitions}
\declaretheorem[name=Remark]{remark}
\crefname{remark}{Remark}{Remarks}
\theoremstyle{plain}

\usepackage{xparse}

\def\eqref#1{equation~\ref{#1}}
\def\1{\bm{1}}
\def\onevec{\mathbf{1}}

\DeclareMathAlphabet{\mathsfit}{\encodingdefault}{\sfdefault}{m}{sl}
\SetMathAlphabet{\mathsfit}{bold}{\encodingdefault}{\sfdefault}{bx}{n}

\def\gI{{\mathcal{I}}}

\def\gL{{\mathcal{L}}}

\newcommand{\R}{\mathbb{R}}

\DeclareMathOperator*{\argmin}{arg\,min}

\DeclareMathOperator{\sign}{sign}

\newcommand{\ceil}[1]{\lceil#1\rceil}

\newcommand{\abs}[1]{\left \vert#1\right \vert}
\newcommand{\paren}[1]{\left(#1\right)}

\newcommand{\braces}[1]{\left\{#1\right\}}

\newcommand{\transpose}{^\mathsf{\scriptscriptstyle T}}
\newcommand{\inverse}{^{\scriptscriptstyle -1}}

\DeclareMathOperator{\dom}{dom}			% domain of a function
\DeclareMathOperator{\inte}{int}		% interior of a set
\DeclareMathOperator{\diag}{diag}		% diagonal matrix
\DeclareMathOperator{\epi}{epi}		% epigraph

\newcommand{\w}[1]{w_{#1}}
\newcommand{\m}[1]{m_{#1}}
\newcommand{\n}[1]{n_{#1}}

\def\loss{\gL}

\newenvironment{sproof}{%
  \proof}{\endproof}

\usepackage{lastpage}
\jmlrheading{23}{2022}{1-\pageref{LastPage}}{1/21; Revised 5/22}{9/22}{21-0000}{Author One and Author Two}

\makeatletter
\newcommand{\myref}[1]{\cref{#1}\mynameref{#1}{\csname r@#1\endcsname}}
\newcommand{\Myref}[1]{\Cref{#1}\mynameref{#1}{\csname r@#1\endcsname}}

\def\mynameref#1#2{%
  \begingroup
    \edef\@mytxt{#2}%
    \edef\@mytst{\expandafter\@thirdoffive\@mytxt}%
    \ifx\@mytst\empty\else
    \space(\nameref{#1})\fi
  \endgroup
}
\makeatother

\begin{document}

\title{Gradient Descent on Logistic Regression with Non-Separable Data and Large Step Sizes}

\author{%
  \name Si Yi Meng \email sm2833@cornell.edu \\
  \addr Department of Computer Science\\
  Cornell University \\
  Ithaca, NY, USA
  \AND
  \name Antonio Orvieto \email antonio@tue.ellis.eu \\
  \addr ELLIS Institute Tübingen\\
  MPI for Intelligent Systems\\
  Tübingen AI Center, Germany
  \AND
  \name Daniel Yiming Cao \email dyc33@cornell.edu \\
  \addr Department of Computer Science\\
  Cornell University \\
  Ithaca, NY, USA
  \AND
  \name Christopher De Sa \email cmd353@cornell.edu \\
  \addr Department of Computer Science\\
  Cornell University\\
  Ithaca, NY, USA
}

\maketitle
\addtocontents{toc}{\protect\setcounter{tocdepth}{-1}}

\begin{abstract}
  We study gradient descent (GD) dynamics on logistic regression problems
  with large, constant step sizes. 
  For linearly-separable data, it is known that GD converges 
  to the minimizer with arbitrarily large step sizes, 
  a property which no longer holds when the problem is not separable.
  In fact, the behaviour 
  can be much more complex --- a sequence of period-doubling bifurcations 
  begins at the critical step size $2/\lambda$, where $\lambda$ is the 
  largest eigenvalue of the Hessian at the solution.
  Using a smaller-than-critical step size guarantees convergence if initialized nearby the solution: but does this suffice globally?
  In one dimension, we show that a step size less than 
  $1/\lambda$ suffices for global convergence.
  However, for all step sizes between $1/\lambda$ and the critical step size $2/\lambda$, one can 
  construct a dataset such that GD converges to a stable cycle. 
  In higher dimensions, this is actually possible even for step sizes 
  less than $1/\lambda$. Our results show that although 
  local convergence is guaranteed for 
  all step sizes less than the critical step size,
  global convergence is not, and GD may instead converge to a cycle depending on the initialization. 
\end{abstract}

\section{Introduction}
\label{sec:intro}
Logistic regression is one of the most fundamental methods for 
binary classification. 
Despite being a linear model, 
logistic regression and its multi-class generalization 
play a significant role in deep learning, appearing in tasks like 
model fine-tuning. 
Given features $x_i\in\R^d$ and binary labels $y_i=1$ or $-1$, the 
goal is to find a linear classifier
\begin{align}
	\label{eq:logreg-loss}
	w^*\in \argmin_{w\in\R^d} \loss(w) \coloneqq \frac{1}{n}\sum_{i=1}^n \log\paren{1+\exp(-y_i w\transpose x_i)}.
\end{align}
Since $w^*$ generally does not admit a closed-form expression, 
iterative methods such as Gradient Descent (GD) are typically used.
GD solves this problem by iterating 
\begin{align}
	\label{eq:gd-update}
	w_{t+1} \gets w_t - \eta_t \nabla \mathcal{L}(w_t),
\end{align}
where $\eta_t$ is the step size. The objective $\loss$
is convex and $L$-smooth, meaning in this case the largest eigenvalue of 
$\nabla^2\loss$ is globally upper bounded by $L$.
Classical optimization theory guarantees that a 
constant step size $\eta_t = \eta < 2/L$ is sufficient for 
GD to converge to $w^*$ for any initialization \citep{nesterov2018lectures}.

Recently, there has been a line of interesting discoveries 
on the behavior of GD, particularly for logistic regression. 
For linearly-separable data, 
\citet{soudry2018implicit} showed that GD with the 
$2/L$ step size
converges to the maximum-margin separator. In fact, this holds true
for \emph{any step size} $\eta >0$ \citep{wu2023implicit}. 
An intuitive explanation is that 
if the data is separable, 
$w^*$ is attained at infinity, 
and so $w_t$ converges in the maximum-margin direction 
but diverges in magnitude. This result shows that the $2/L$ 
condition on the step size is unnecessarily conservative for logistic regression.

If the data is not linearly-separable, 
the objective is strictly convex as long as 
the features have full rank, thus the unique minimizer 
$w^*$ is finite. For this reason, one can not expect convergence 
under an arbitrarily large step size: indeed, classical dynamical 
systems theory \citep{strogatz2018nonlinear,sayama2015introduction} 
shows that $w^*$ becomes unstable when the step size $\eta > 2/\lambda$, 
where $\lambda$ is the largest eigenvalue of the Hessian of $\mathcal{L}$
at $w^*$. Compared to the global upper bound on the Hessian given by 
$L$-smoothness, i.e. $\nabla^2 \loss(w)\preceq LI$ for all $w\in\R^d$, 
the local smoothness constant $\lambda$ at $w^*$ can be much smaller, 
resulting in the much larger step size threshold of $2/\lambda$.
A natural question to ask is whether this is the only barrier in 
the non-separable case: 
can we still guarantee convergence, as in the separable setting, 
for all ``large'' step sizes $\eta$ between $2/L$ and the ``critical'' 
step size of $2/\lambda$? And what happens at even larger $\eta$?
Large step sizes are interesting because they can often lead to 
faster convergence, both for logistic regression in the separable case 
\citep{axiotis2023gradient,wu2024large} and for more general problems where 
step sizes can be above $2/L$ (but not necessarily constant)
\citep{altschuler2023acceleration1,grimmer2023accelerated,mishkin2024directional,oymak2021provable,wu2023reg}.
We also don't know \emph{a priori} whether the data is separable, 
so it would be helpful to gain a better understanding of what happens when 
we push the step size beyond $2/L$. 
The large step size regime has also been studied 
for training deep neural networks,
where it gives rise to the Edge-of-Stability phenomenon
\citep{cohen2022gradient} and can cause spikes or catapults in the
initial steps of optimization \citep{lewkowycz2020large,zhu2023catapults}.
Unstable convergence in this setting has been theoretically studied by 
\citet{ahn2022understanding,wang2023good}.

In this paper, we study the behaviour of GD on logistic regression in the 
non-separable setting, where the step size is constant but 
potentially much larger than $2/L$.
We begin by showing that as $\eta$ increases past the critical step size $2/\lambda$, 
GD follows a route to chaos characterized by a cascade of period doubling.

\begin{minipage}{0.55\textwidth}
	If the problem is one-dimensional, we prove that 
	$\eta=1/\loss''(w^*)=1/\lambda$ is the largest step size for which 
	GD converges globally to $w^*$, and the rate is linear after a finite 
	number of iterations. Beyond this step size, we construct a 
	dataset on which GD can instead converge to a cycle. 
\end{minipage}
\hfill
\begin{minipage}{0.44\textwidth}
	\captionof{table}{Provable dynamics.}
	\begin{center}
		\begin{tabular}{@{}lll@{}}
		\toprule
				& $\gamma \leq 1$ & $1 < \gamma \leq 2 $  \\ \midrule
		$d=1$ & Convergence & Cycles exist		       \\
		$d>1$ & \multicolumn{2}{c}{Cycles exist}       \\ \bottomrule
		\end{tabular}
	\end{center}
\end{minipage}

Finally, for higher dimensional problems, we show that any step size of the form 
$\eta = \gamma/\lambda$ for $\gamma\in(0,2]$ can result 
in convergence to a cycle. 
Interestingly, these are not just an algebraic property of the logistic 
regression objective: in fact, our results hold for any loss functions 
structurally similar to the 
logistic loss, in that they look like a ReLU in the large. 
\section{Background}
Non-separable logistic regression problems differ 
from the separable setting largely due to the location of the minimizer. 
If the data is non-separable, the objective is strictly convex 
in the subspace spanned by the features $\{x_i\}$, and the solution 
is no longer attained at infinity. 
GD on logistic regression is essentially a discrete time nonlinear
dynamical system for which $w^*$ is a fixed point.
A necessary condition for $w^*$ to be locally (linearly) stable is 
to have a step size smaller than 
$2/\lambda$, where $\lambda$ is the \emph{largest eigenvalue of the Hessian at $w^*$}
\citep{strogatz2018nonlinear,sayama2015introduction}. Local stability means GD 
converges to $w^*$ 
when we initialize close enough to it. On the other hand, if we view 
the problem from a convex optimization perspective, a sufficient condition 
for global convergence is to require $\eta < 2/L$, where $L$ 
is a \emph{global upper bound on the Hessian}, 
which can be much larger than $\lambda$. 
One can relax the $2/L$ requirement by leveraging the generalized 
self-concordance property of logistic regression \citep{bach2010self}, 
yielding an instance-dependent large step size.
One can also use step sizes that depend on
the local smoothness (Hessian around the current iterate), 
as do \citet{ji2019implicit} for non-separable logistic regression and 
\citet{mishkin2024directional} for general convex problems. 
However, these step sizes are either still strictly 
below the $2/\lambda$ stability threshold, or still effectively require that 
the objective decreases monotonically, 
which is not guaranteed for $\eta > 2/L$.
For logistic regression,
\citet{liu2023nonseparable} created a 
two-example dataset with identical features but 
opposite labels, on which GD can enter a stable period-$2$ cycle when 
the step size is greater than a critical value. Unfortunately, this 
critical step size coincides with both $2/L$ and $2/\lambda$ due to the 
degeneracy of the dataset.
which means there is still little known 
about what happens when there is a non-trivial 
gap between $2/L$ and $2/\lambda$.

Beyond linear classification problems, period-doubling bifurcations 
and chaos in GD dynamics under large step sizes have been observed in 
many problem settings. 
For least squares problems, \citet{van2012chaotic} showed that  
occasionally taking very large step sizes can lead to much faster convergence.
But if these step sizes are too large, GD can behave chaotically. 
Beyond linear models, 
\citet{chen2022gradient} gave sufficient conditions for a period-$2$ cycle to 
exist for one-dimensional loss functions, but the study is mostly restricted to 
the squared loss. 
\citet{zhu2022quadratic,zhu2023catapults} empirically 
studied non-monotonic convergence of GD under the critical step size 
$\eta < 1/\lambda_0$, where $\lambda_0$ is the largest eigenvalue of the Hessian 
at initialization. Their studies apply to general neural networks 
but are also limited to the squared loss. 
\citet{chen2023from} proved that under restrictive conditions 
on the input data and architecture, GD on neural networks with nonlinear activations 
boils down to a one-dimensional cubic map that can behave periodically or 
chaotically. Once again, these results only apply to the 
squared loss. Under general loss functions and model architectures, 
\citet{ahn2022understanding} gave intuitions to when GD can converge 
under unstably large step sizes, while \citet{danovski2024dynamical} 
observed stable oscillation and chaos for neural network training. 
For objectives with a multiscale structure, \citet{kong2020stochasticity}
gave sufficient conditions for GD to provably exhibit 
Li-Yorke chaos \citep{li2004period}.
% Unlike GD on quadratic objectives
% a step size above $2/L$ leads to divergence almost surely, nonlinearity in 
% the dynamics can sometimes keep iterates 
% bounded and allow GD to converge non-monotonically.

In the deep learning literature, the Edge of Stability (EoS) phenomenon 
\citep{cohen2022gradient} refers to a rich set of behaviors observed 
when training neural networks with large step sizes.
Specifically, it has been observed that when training neural networks,
the largest eigenvalue of the Hessian, also referred to as the sharpness, 
often hovers right at, or even above $2/\eta$, 
while the objective continues to decrease. 
\citet{zhu2022understanding} illustrated this 
phenomenon using a minimalist 4-parameter scalar network with the quadratic loss, 
where GD iterates initially oscillate, then de-bifurcate, 
leading to convergence at an EoS minimum. 
Another motivation for studying GD step sizes in the EoS
regime is that large step sizes can be crucial in 
learning the underlying representations of the problem.
For instance, \citet{ahn2022learning} showed that in a 
sparse-coding setup, one can only learn the bias term 
necessary for recovery by dialing up the step size
into the unstable regime.
While there are many more works studying the EoS regime for non-convex problems 
\citep{song2023trajectory,kreisler2023gradient,wang2023good,lu2023benign}, 
we believe it is useful to take a step back and closely examine 
what exactly happens on just linear models, especially 
for the logistic loss which seems to be under-explored. 
We show that stable cycles can occur under the critical step size
$2/\lambda$, and illustrate precisely how these cycles arise. 
\section{Period-doubling bifurcation and chaos}
\label{sec:bifurcation}

When the examples are not linearly-separable, that is, for all $w\in\R^d$, 
there exists $i \in \{1,\dots,n\}$
such that $y_iw\transpose x_i \leq 0$, the logistic regression objective 
$\loss$ in \cref{eq:logreg-loss} is strictly convex (\cref{fact:logreg-strict-cvx}), 
as long as the $x_i$'s span $\R^d$. The solution $w^*$ 
is necessarily unique and finite, so we can simply run GD with increasing 
step sizes to examine its convergence properties.
In \cref{fig:real-main}, we see that GD is convergent for small step sizes, 
up to a point at which a period-$2$ cycle emerges. As we continue to increase 
the step size beyond this point, a sequence of period-doubling bifurcation 
occurs, and GD converges to cycles of longer periods. Eventually, 
this period-doubling cascades into chaos. 
\begin{figure}
	\centering \includegraphics[width=0.8\linewidth]{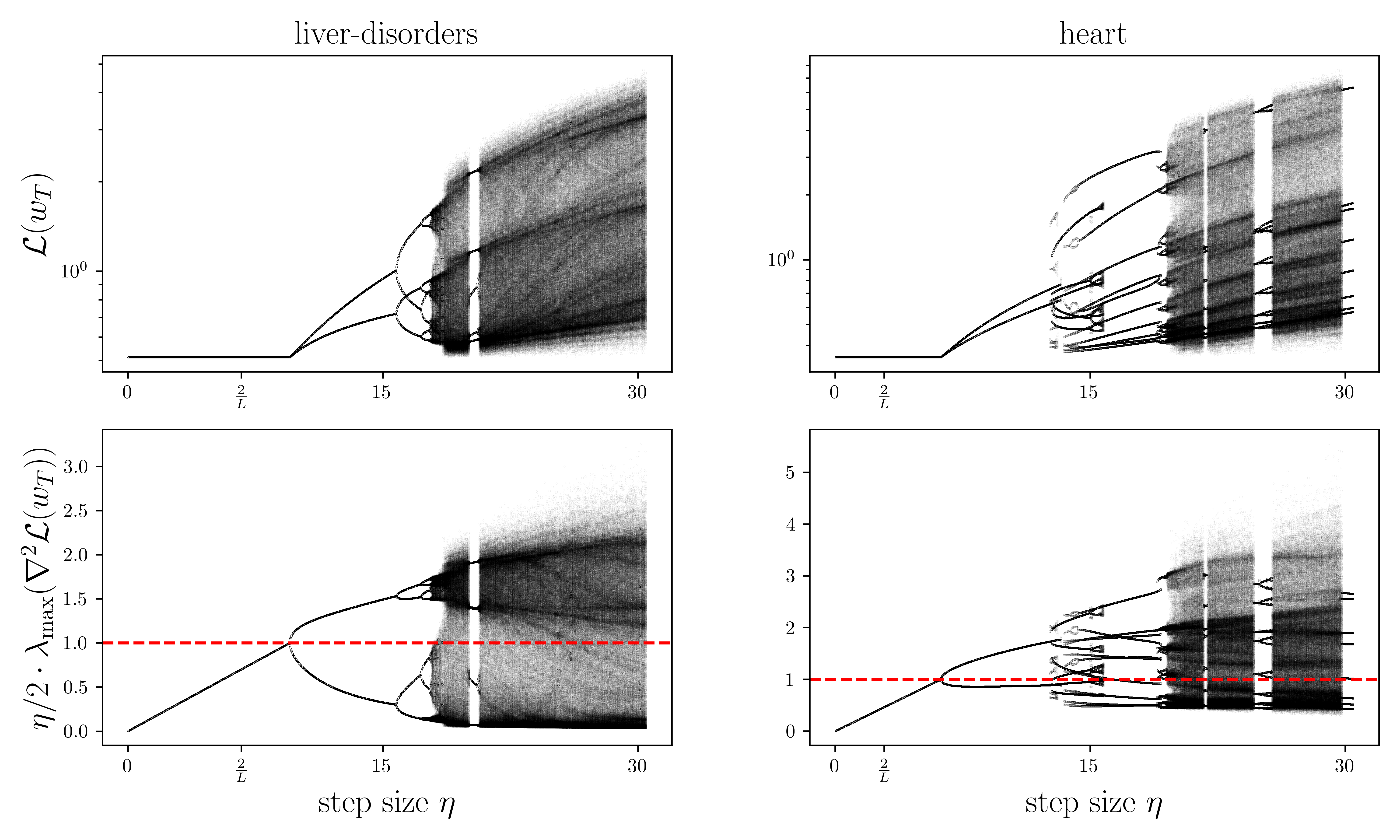}
	\caption{
		Bifurcation diagrams on two binary classification datasets from 
		the \texttt{LIBSVM} repository \citep{chang2011libsvm}, both are non-separable. 
		For each step size, we run GD for $T= 5\cdot 10^5$
		iterations with $1024$ different random initializations 
		of varying scales. 
		Each point corresponds to the loss (first row) or the (scaled) 
		largest eigenvalue of the 
		Hessian (second row) evaluated at the final iterate $w_T$. 
		When multiple points are 
		visible, GD either converged to a cycle or is 
		chaotic under that step size.
	}
	\label{fig:real-main}
\end{figure}
The ``discontinuous'' regions in the bifurcation diagram on the \texttt{heart} 
dataset (right panel of \cref{fig:real-main}), around $\eta=15$,
correspond to different cycles arrived at 
when starting at initializations of various scales, and can be 
reproduced on a synthetic dataset 
shown in \cref{fig:different-cycles}. The discontinuous regions are also 
present (around $\eta=35$). We then ran GD with $\eta=35$ 
using two arbitrary initializations that differ in scale, 
and observe that GD can converge to 
different cycles of different periods. 
\begin{figure}
	\centering \includegraphics[width=0.95\linewidth]{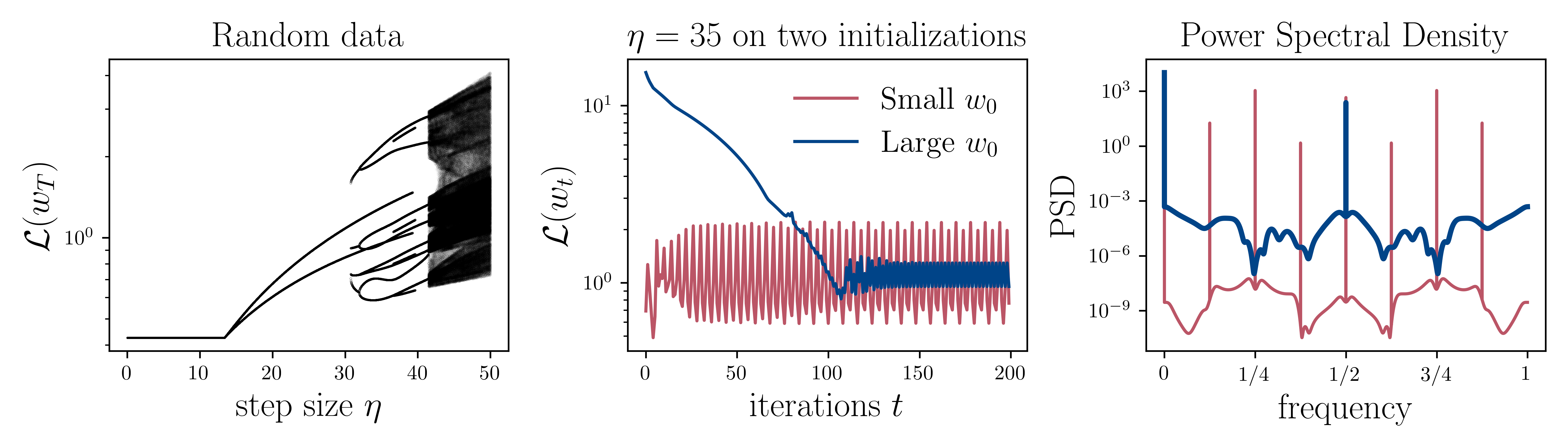}
	\caption{
		On the left is the bifurcation diagram on a synthetic dataset 
		with $n=12$ and $d=4$. The $x_i$'s are generated from 
		the standard Gaussian distribution, with uniformly random labels. 
		In the middle, we plot the loss at each GD iteration when ran 
		with $\eta=35$, for two different initializations 
		$w_0=100\cdot \onevec$ and $w_0=0.001\cdot \onevec$, 
		where $\onevec$ is the all $1$'s vector. 
		On the right, we compute the power spectral density of the 
		losses over $t$, showing a period-$2$ cycle for the large 
		initialization, while the small initialization run converged
		to a period-$8$ cycle. 
	}
	\label{fig:different-cycles}
\end{figure}
The fact that GD undergoes period-doubling bifurcation on the logistic 
regression objective is perhaps not so surprising --- this phenomenon 
has long been observed and studied on numerous nonlinear maps 
\citep{strogatz2018nonlinear}. Unlike GD on linear regression which results in 
a linear map in discrete time, the sigmoid link in the gradient adds a 
nonlinearity that gives rise to a much richer 
spectrum of behaviour. What the bifurcation diagrams can help us see is 
at what step size we cease to have convergence to $w^*$. To further our 
understanding, consider 
the gradient and Hessian of \cref{eq:logreg-loss}:
\begin{align}
	\nabla \loss(w) =  \frac{1}{n} \sum_{i=1}^n \sigma(-y_iw\transpose x_i) (-y_ix_i),
	\quad 
	\nabla^2 \loss(w) = \frac{1}{n} \sum_{i=1}^n \sigma'(-y_iw\transpose x_i) x_ix_i\transpose,
\end{align}
where $\sigma(z)= 1/(1+\exp(-z))$ is the sigmoid function. 
We will use $\lambda$ to denote $\lambda_{\max}(\nabla^2\loss(w^*))$.
As the Hessian is maximized at $w=0$, 
the global upper bound on the Hessian is given by 
$L = \lambda_{\max}( XX\transpose/(4n) )$, 
where $X\in \R^{n\times d}$ is the feature matrix. 
One interesting observation from \cref{fig:real-main} lies in the second row, 
where we plot the final values of $\lambda_{\max}(\nabla^2\loss(w_T))$
scaled by $\eta/2$. 
For all step sizes below $2/L$, this value is clearly convergent, as $w_t$ converges to $w^*$ regardless 
of initialization.
It appears that GD remains convergent beyond $\eta=2/L$
until $\eta = 2/\lambda$, at which period-$2$ cycles begin to appear. 
This is reasonable as $w^*$ is the fixed point of the GD map 
$T(w) = w-\eta \nabla \loss(w)$, and first-order stability of $w^*$ is guaranteed 
if the eigenvalues of the Jacobian of $T$ 
\begin{align*}
	J_T(w) = I - \eta \nabla^2\loss (w)
\end{align*}
lie strictly within the unit circle \citep[Chapter 5.7]{sayama2015introduction}. 
As a result, $\eta < 2/\lambda$ 
is a necessary condition for GD to converge to $w^*$ in a neighborhood of $w^*$. 
Next, we illustrate on a toy dataset that the gap
between $2/L$ and $2/\lambda$ can grow arbitrarily large, 
and that analyzing the cycles is very challenging.

\subsection{A toy dataset} 
Consider $n\geq 2$ examples such that $y_i=1$ 
for all $i$. Let $v$ be an arbitrary point on the $d$-dimensional unit sphere. 
The dataset consists of $n-1$ copies of $v$, and a single copy of $-v$. 
Clearly, this dataset is not separable by any linear classifier that goes 
through the origin.  
The gradient and Hessian simplify to (see \cref{sec:toy-dataset-calculations})
\begin{align*}
	\nabla \loss(w) = \frac{1}{n} \paren{ n\cdot \sigma(v\transpose w) - (n-1) }v 
	\quad \text{and} \quad 
	\nabla^2 \loss(w) = \sigma'(v\transpose w) \cdot vv\transpose.
\end{align*}
Setting the gradient to $0$ gives us $\sigma(v\transpose w^*) = \frac{n-1}{n}$. 
Since the largest eigenvalue of $vv\transpose$ is $1$,  
\begin{align*}
	\lambda = \sigma(v\transpose w^*)(1-\sigma(v\transpose w^*)) = \frac{n-1}{n^2},
\end{align*}
while $L=1/4$. So for large $n$, the gap between 
the step sizes $2/L$ and $2/\lambda$ grows quickly.  

This toy dataset can also help us get a sense of how difficult it is to analyze 
these cycles. Note that this problem is degenerate in the sense that $X$ is 
rank-$1$, so the resulting objective is not strictly-convex and there exists a subspace 
of minimizers. Instead, we analyze the associated GD update on the probability 
space. 
For $i=1,\dots,n$, let $p_{t,i} \coloneqq \sigma(-y_i w_t\transpose x_i)$. 
A recurrence relation for $p_{t,i}$ can be derived as follows --- simply 
take the inner product with $-y_i x_i$
on both sides of the GD update 
\labelcref{eq:gd-update} and apply the sigmoid function, 
giving us 
\begin{align}
	p_{t+1,i} &\coloneqq \sigma\paren{ \sigma\inverse(p_{t,i}) - \frac{\eta}{n} y_i \paren{ {\sum_{j=1}^n} y_j p_{t,j} x_j\transpose} x_i   }\quad i=1,\dots, n,
\end{align}
where $\sigma\inverse(p) = \ln\paren{ \nicefrac{p}{1-p}} $ is the logit function,
and $p_{0,i} = \sigma(-y_iw_0\transpose x_i)$. 
On the toy dataset, this update can be simplified into
\begin{align}
	\label{eq:one-d-toy-map}
	p_{t+1,n} &= \sigma\paren{\sigma\inverse(p_{t,n}) - \frac{\eta}{n} \paren{ p_{t,n} - (n-1)(1-p_{t,n}) }} \nonumber \\
	p_{t+1,i} &= 1 - p_{t+1,n} \quad \forall i=1,\dots,n-1.
\end{align}
The bifurcation diagrams of GD on this toy dataset are presented in 
\cref{fig:toy_one_class}.
For $n=2$ and $\eta\geq 2/\lambda =8$, the two points of oscillation 
are given by the two values of
\begin{align}
	\label{eq:two-examples-oscillation}
	p_n = \frac{1}{2}\paren{ h\inverse\paren{\frac{\eta}{8} } + 1 }, 
\end{align}
where $h(z) = \tanh\inverse(z)/z$ (see \cref{prop:two-examples-oscillation} 
in the Appendix for the derivation). Since $h(z)\geq 1$ for all $z$, the 
period-$2$ point $p_n$ is only defined when $\eta/8\geq 1$, as expected.
This shows that even with $n=2$ on this trivially-constructed dataset, 
computing the two points of oscillation is a nontrivial task as 
$h\inverse$ is not even an elementary function. 
\section{Technical setup}
\label{sec:prelim}

We now provide the technical setup for analyzing convergence and cycles 
in the large step size setting.
Consider the linear classification problem with loss function $\ell$ of finding 
\begin{align}
	\label{eq:classification-obj}
	w^* \in \argmin_{w\in\R^d} \loss(w) = \smash{ \frac{1}{n}\sum_{i=1}^n \ell(-y_i w\transpose x_i) }
\end{align}
where $y_i \in \braces{-1, 1}$, $x_i\in\R^d$, 
and $\ell$ is the loss function on a single example.
We use $\lambda$ to denote $\lambda_{\max}(\nabla^2\loss(w^*))$.
We are particularly interested in the case where the data is not 
linearly-separable. Concretely, we assume that for 
all $w\in\R^d$, there exists $i\in\braces{1,\dots,n}$ such that 
$y_i w\transpose x_i \leq 0$. This implies that $\nabla^2\loss(w) \neq 0$
for all $w$.
While our motivation is to study the behaviour of large step size GD 
on logistic regression, 
our results will be stated in terms of general loss functions 
$\ell$ satisfying the following set of assumptions.
This set of assumptions essentially requires that $\ell$
looks like a ReLU when zoomed out.

\begin{minipage}{0.59\textwidth}
	\begin{restatable}{assmpt}{lossassumptions}
		\label{assmpt:individual-loss}
		The individual loss $\ell:\R\to\R_+$ is
		\begin{enumerate}[leftmargin=*]
			\item three-times continuously-differentiable and strictly convex, 
			and $\lim_{z\to -\infty} \ell(z) = 0$, 
			\item $\lim_{z\to -\infty} \ell'(z) = 0$ and $\lim_{z\to \infty} \ell'(z) = 1$, and
			\item $\ell''$ is increasing on $(-\infty, 0]$ and decreasing on $[0, \infty)$; furthermore, it decays rapidly as
            \begin{align*}
				\lim_{\epsilon\to 0} \; \frac{1}{\epsilon^2} \ell''\paren{ \frac{1}{\epsilon}} = 0.
			\end{align*} 
		\end{enumerate}		
	\end{restatable}
	The limit of $\ell$ and the upper bound on $\ell'$ 
	can both be generalized to any finite positive value. 
\end{minipage}
\hfill%
\begin{minipage}{0.38\textwidth}
	% \vspace*{-0.5em}
	\begin{figure}[H]
		\centering \includegraphics[width=\textwidth]{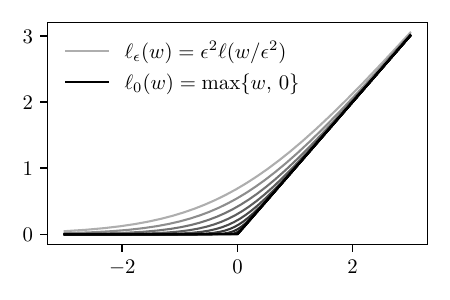}
		\vspace*{-1.7em}
		\caption{The limit using the logistic loss as an example. 
		Darker grey represents smaller values of $\epsilon$. }
		\label{fig:limit-is-relu}
	\end{figure}
	\vfill%
\end{minipage}
\begin{restatable}{cor}{limitisrelu}
	\label{cor:limit-is-relu}
	\cref{assmpt:individual-loss} implies that 
	$\lim_{\epsilon\to 0} \epsilon^2 \ell\paren{\frac{z}{\epsilon^2}} = \max\braces{z,\, 0}$.
\end{restatable}
The logistic loss $\ell(z) = \log(1+\exp(z))$ and the 
squareplus function $\ell(z) = 0.5(\sqrt{4+z^2}+z)$ 
\citep{barron2021squareplus} both satisfy \cref{assmpt:individual-loss}, as 
verified in \labelcref{sec:assumptions-on-ell-verification}. 
Outside of these two losses, there are few commonly-used machine learning
loss functions that satisfy our assumptions. 
The purpose of stating our results in this way is to 
emphasize that our results are not an algebraic consequence of 
the logistic loss, but rather, a consequence of its structural properties.

\section{One-dimensional case}
\label{sec:1d-case}

As discussed in the previous section, although $\eta < 2/\lambda$ is a 
necessary condition for convergence, it is not sufficient.
Specifically, it only guarantees \emph{local} convergence to 
$w^*$ when we initialize within a neighborhood of $w^*$.
From \cref{fig:real-main}, it appears that we do converge globally for all 
$\eta < 2/\lambda$, as the initializations we used vary in scale.
However, as we show later, this is not the case for all datasets.
The first natural question to 
ask is \emph{what is a sufficient condition on the step 
size to guarantee global convergence?}
As it turns out, when $d=1$, this step size is given by $\eta \leq 1/\lambda$.

\begin{restatable}{thm}{onedimconvergence}
	\label{thm:1d-convergence}
	Suppose \cref{assmpt:individual-loss} holds 
	for the classification problem in $1$ dimension 
	with non-separable data. %, that is, $w^*$ is finite.
	Then the GD iterates 
	$\braces{w_t}$ under $\eta \leq 1/\loss''(w^*)$ converge to $w^*$ 
	for any initialization $w_0$. 
	When $\eta = 1/\loss''(w^*)$, let 
	$\bar \tau \coloneqq 1 + \max\braces{0, \ceil{ -w_0 \cdot \sign(w^*) \cdot \loss''(w^*)/\abs{\loss'(0)} }}$.
	Then there exists $\rho\in(0,1)$ such that for all $t\geq \bar \tau$, we have 
	$(w_{t+1} - w^*)^2 \leq \rho (w_t - w^*)^2$.
\end{restatable}
Convergence is straightforward. Suppose $w^*>0$. The interval $\gI=[w^*, +\infty)$ is an 
invariant set under the step size requirement, that is, 
if $w_t\in\gI$, then $w_{t+1}\in\gI$. Furthermore, 
$w_t$ converges to $w^*$ on $\gI$.
So either we converge directly 
in $\gI$, or we initialize in $(-\infty, w^*)$,
from which we either approach $w^*$ from the left, or cross over into $\gI$ ,
within which convergence is guaranteed.
The detailed proof including the rate can be found in \cref{sec:1d-convergence}.
The next result shows that if we increase the step size beyond $1/\loss''(w^*)$,
global convergence is no longer guaranteed. 

\begin{restatable}{thm}{onedimcycle}
	\label{thm:1d-convergence-to-cycle}
	For any $\ell$ for which \cref{assmpt:individual-loss} holds
	and any $\gamma\in (1,2]$, there exists a 1-dimensional 
	non-separable classification 
	problem of the form \labelcref{eq:classification-obj}, on which 
	there exists a GD trajectory under the step size 
	$\eta = \gamma/\loss''(w^*)$
	that converges to a cycle of period $k>1$.
\end{restatable}

It is worth emphasizing that \cref{thm:1d-convergence-to-cycle} does not imply  
a cycle is possible for every dataset. It only shows that if we were to 
first pick a $\gamma\in(1,2]$, then we can construct 
a dataset such that 
GD with step size $\eta = \gamma/\loss''(w^*)$ \emph{can} converge to a cycle. 
This construction necessarily implies $\eta > 2/L$
where $L=\loss''(0)$ is the global smoothness constant, as otherwise 
we would have global convergence. 
Stability of the cycle coexists with stability of the solution $w^*$. 
So depending on the initialization, 
GD may still converge to $w^*$ with the same step size.
Moreover, given an arbitrary dataset, it 
is not always easy to verify whether 
a cycle exists before hitting the $2/\lambda$ step size. 
In fact, when we ran GD using many different 
initializations of varying scales in on real datasets (\cref{fig:real-main}), 
we did not see a clear cycle emerging at all until 
the stability threshold of $\eta = 2/\lambda$.
The intuitions and main proof steps are as follows.
\begin{sproof}
	Fix a $\gamma\in(1,2]$, we begin by constructing a dataset 
	where the $x_i$'s are copies of $1$'s and $-1$'s, with all $1$'s label.
	This dataset corresponds to a loss $\loss$ such that 
	the minimizer is $w^*>0$ without loss of generality. On this $\loss$, 
	it can be shown that a trajectory of the form 
	\begin{align}
		w_{k-1} < 0 < w^* < w_{k-2} < \dots < w_0
	\end{align}
	exists. Each iterate is generated via the GD update 
	on $\loss$ with step size $\eta = \gamma / \loss''(w^*)$, illustrated 
	in the left panel of \cref{fig:perturbed-loss-1d}.
	Furthermore, we can ensure that 
	$w^* < w_k < w_0$.  

	Then, consider perturbing $\loss$ by adding a ReLU to it, giving us
	\begin{align}
		\loss_{\text{perturbed}}(w) = \loss(w) + \frac{w_0 - w_k}{\eta} \max\braces{ -w , \, 0}.
	\end{align}
	Observe that the ReLU does not contribute any additional curvature 
	on top of $\loss$, and it also does not contribute any gradient to 
	all points $w>0$. Therefore, the minimizer remains the same.
	So if we were to run GD starting at the same $w_0$ 
	on $\loss_{\text{perturbed}}$ with the step size 
	$\eta_{\text{perturbed}} = \gamma / \loss''_{\text{perturbed}}(w^*) = \eta$, 
	the new trajectory would coincide with the original trajectory up to 
	step $k$. Let $\tilde w_t$ denote the iterates of this new trajectory 
	for $t \geq 0$. At point $\tilde w_{k-1} = w_{k-1} < 0$, the ReLU becomes active.
	Applying one GD update at this point leads to
	\begin{align}
		\tilde w_k &= \tilde w_{k-1} - \eta \loss'_{\text{perturbed}}(\tilde w_{k-1}) \\
		&= \tilde w_{k-1} - \eta \loss'(\tilde w_{k-1}) + \eta \frac{w_0 - w_k}{\eta} \nonumber \\
		&= w_k + w_0 - w_k  \nonumber \\
		&= w_0, \nonumber 
	\end{align}
	resulting in a cycle. As shown in the second panel of 
	\cref{fig:perturbed-loss-1d}, we have effectively added enough gradient 
	to the left of $0$ such that $\tilde w_{k+1}$ gets 
	kicked right back to where we started. 
	\begin{figure}%[H]
		\centering \includegraphics[width=0.9\linewidth]{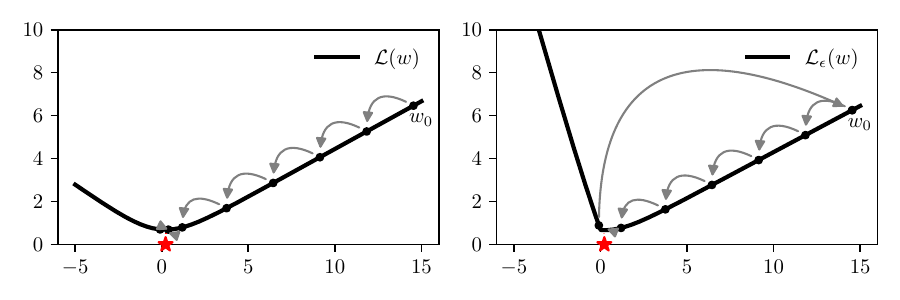}
		\vspace*{-1.2em}
		\caption{Illustration of our cycle construction. On the left, $\loss$ is 
			the logistic loss with data consisting of $250$ copies of $x_i=1$
			and $200$ copies of $x_i=-1$. On top of this dataset we add $15$ 
			copies of $x_i=70$ to get $\loss_\epsilon$ on the right. 
			The red star marks the minimizer $w^*$ and $w^*_\epsilon$ of the 
			respective objective.
			Starting at the same $w_0$ and using the same $\gamma=1.5$, 
			GD on $\loss$ with $\eta = \gamma/\loss''(w^*)$ converges to the 
			minimizer, while GD on $\loss_\epsilon$ 
			with $\eta_\epsilon = \gamma/\loss_\epsilon''(w^*_\epsilon)$
			converges to a period-$7$ cycle.}
		\label{fig:perturbed-loss-1d}
	\end{figure}
	However, adding a ReLU to a loss function of the form 
	\labelcref{eq:classification-obj} does not automatically constitute 
	a valid classification problem. 
	We need to show that the kick can be achieved equivalently by 
	adding more examples to the dataset, yielding 
	an objective of the finite-sum form \cref{eq:classification-obj}.
	Recall that by \cref{cor:limit-is-relu},
	this ReLU is the limit of the function
	$\ell_{\epsilon}(w) = \epsilon^2 \ell\paren{ - w/\epsilon^2}$
	as $\epsilon\to 0$ (\cref{fig:limit-is-relu} and \cref{cor:limit-is-relu}). 
	This allows us to define a continuous perturbation 
	\begin{align}
		\loss_\epsilon(w) = \loss(w) + \frac{w_0 - w_k}{\eta} \ell_\epsilon(w).
	\end{align}
	It remains to invoke the implicit function 
	theorem to argue that there exists a 
	non-zero $\epsilon$ such that $\epsilon^2 (w_0-w_k) / \eta >0$ is 
	rational. This implies we can obtain a loss function equivalent to adding integer copies of  
	the example $x=1/\epsilon^2$ of label $y=1$ 
	to the original dataset. This ensures that $\loss_\epsilon$ 
	indeed corresponds to an objective on a valid binary classification dataset. 
	
	Finally, we argue that this cycle is locally stable.
	This can be achieved by observing that $\loss''$ evaluated 
	at all points on this cycle (except for 
	the one to the left of $w^*$) is strictly smaller than that 
	at $w^*$ (\cref{assmpt:individual-loss}).
    This is sufficient 
	to guarantee that the Jacobian of the $k+1$th iterate GD map 
	has a magnitude strictly less than one; if this is not the case, we 
	can extend the trajectory backwards from $w_0$ to $w_{-1}$, $w_{-2}$, ...,  
	until satisfied.
	Local stability implies that 
	if we initialize sufficiently close to the cycle (rather than the 
	minimizer), GD would converge to this cycle instead.
\end{sproof}
As verifying the assumptions of the implicit 
function theorem can be technical and tedious, 
we defer the full proof to \cref{sec:1d-cycle} and the supplemental materials. 
\begin{remark}
In the one-dimensional setting, the construction in \cref{thm:1d-convergence-to-cycle}
can be adapted to produce stable cycles of any period $k\geq 2$, i.e.,
for all $\gamma\in(1,2]$ and all integer $k\geq 2$, there exists a dataset on which 
there exists a GD trajectory under the step size $\eta=\gamma/\loss''(w^*)$ 
that converges to a cycle of period $k$. The key step lies in constructing 
a base dataset such that the unperturbed trajectory satisfies 
$w_1 < 0 < w^* < w_0$, and $w^*<w_2<w_0$ for that $\gamma$. 
Since $\gamma>1$, these crossovers are almost guaranteed for $w_0$ 
close enough to $w^*$. Then we can perturb this base objective such 
that $w_1$ gets kicked to $w_0$, or to any point 
$w_{-j}>w_0$ such that in $j$ steps, $T^j(w_{-j}) = w_0$. 
Extending the trajectory to the right 
in this manner allows us to create a cycle of arbitrary length $k=2+j$ for $j\geq 0$.
We do not pursue a formal proof as this strengthening is not needed for 
our main conclusions, and the analogous statement 
in higher dimensions for $\gamma\in(0,1]$ remains open.
\end{remark}

In \cref{fig:cycles-w-diff-gammas}, we illustrate that one can 
easily construct a dataset for different values of $\gamma\in(1,2]$
such that GD converges to cycles. 
A pattern to note is that the smaller the $\gamma$, the longer period the 
cycle tends to have, as it takes more steps to move along the ``flatter''
side of the loss before being kicked back. 
Note that convergence to $w^*$ and stable cycles are not necessarily the only two 
events that can happen below the $2/\lambda$ step size. If for some 
dataset there exists a period-$3$ cycle under the step size 
$\eta = \gamma/\lambda$ for some $\gamma<2$,
then a chaotic GD trajectory can also occur for that dataset 
under the same $\gamma$ \citep{li2004period}. 

\begin{figure}
	\hspace*{-1em} \includegraphics[width=\linewidth]{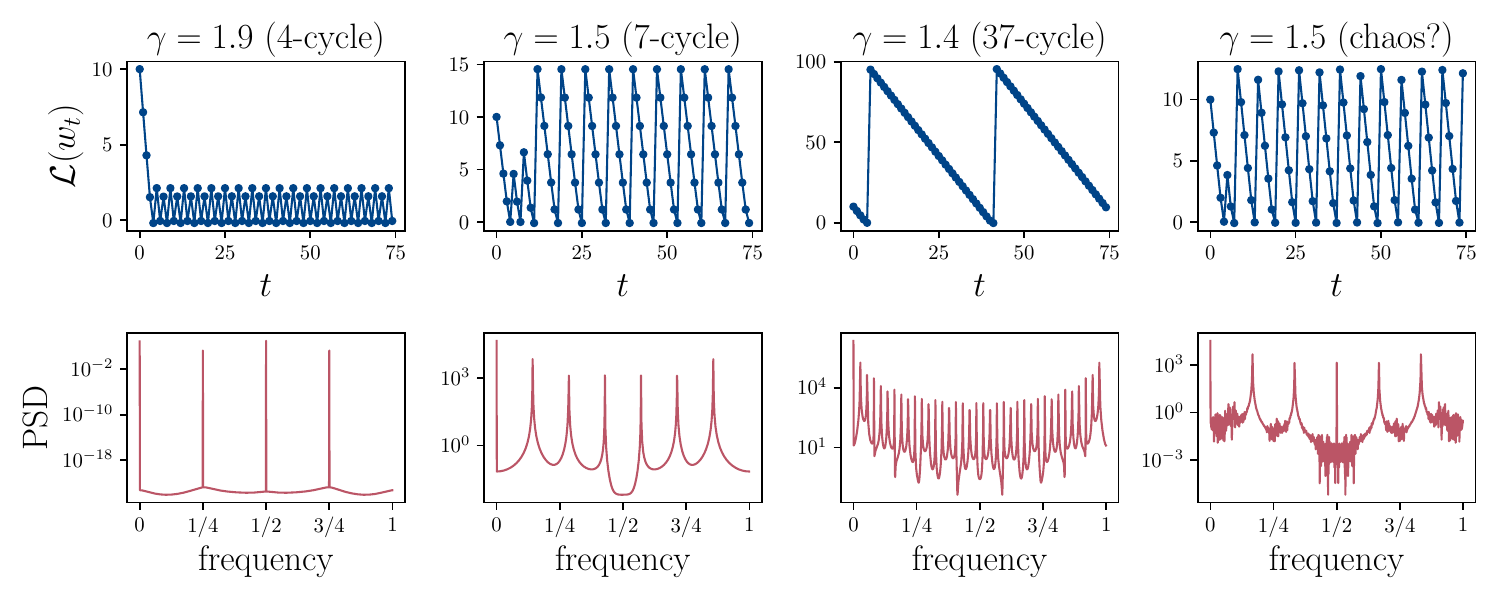}
	\caption{
		For each $\gamma$, we construct a one-dimensional dataset 
		such that GD on this problem 
		with the logistic loss converges to a stable cycle under the step size 
		$\eta=\gamma/\loss''(w^*)$. 
		One exception is the last column which 
		seems to suggest GD can even be chaotic when $\gamma<2$. 
		Figures in the first row show the loss evaluated at successive iterates, 
		while the second row shows the power spectral density of the losses 
		at the last $1024$ iterations. 
		Details on each dataset are in \cref{sec:exps}.
	}
	\label{fig:cycles-w-diff-gammas}
\end{figure}

\section{Higher dimensions}
\label{sec:higher-dimension-case}

In the one-dimensional case, we have shown that under the step size 
$\eta=\gamma/\loss''(w^*)$, convergence is guaranteed
for all $\gamma\leq 1$, while
increasing $\gamma$ beyond $1$ can lead to cycles.
In higher dimensions, 
we no longer have global convergence even 
when $\gamma$ is small. 
As we now show, for all $\gamma\in(0,1]$, a cycle can be  
constructed in $d=2$ using similar techniques. 
Since we can embed the constructed 
dataset in higher dimensions, the result trivially extends to 
any non-linearly-separable classification problem 
satisfying \cref{assmpt:individual-loss} of any dimension $d>1$. 
\begin{restatable}{thm}{twodimcycle}
	\label{thm:2d-convergence-to-cycle}
	For any loss function $\ell$ for which \cref{assmpt:individual-loss} holds
	and any $\gamma\in (0,1]$, there exists a 2-dimensional 
	non-separable classification 
	problem of the form \labelcref{eq:classification-obj}, on which 
	there exists a GD trajectory under the step size 
	$\eta = \gamma/\lambda$
	that converges to a cycle of period $k>1$.
\end{restatable}
The idea of the proof is similar to that of \cref{thm:1d-convergence-to-cycle}.
We briefly discuss the technique and main differences here 
and defer the full proof to \cref{sec:cycle-construction-2d}.
\begin{sproof}
	Fix a $\gamma\in(0,1]$, we construct a base dataset corresponding 
	to a loss $\loss$ such that the minimizer is strictly positive. 
	This base dataset is essentially two 
	1-dimensional datasets lying independently 
	in the two dimensions.
	We show that there exists a trajectory that moves almost in a straight 
	line towards the minimizer, as in 
	\cref{fig:perturbed-loss-2d} (left). Let this trajectory be 
	$w_0, w_1, \dots$, up to $w_{k-2}$, and $T(w) = w - \eta \nabla \loss(w)$ 
	be the GD update on $\loss$. 

	We then perturb $\loss$ by adding two ReLUs to it. Define 
	\begin{align}
		\loss_{\text{perturbed}}(w) = \loss(w) + c_1 \max\braces{-w\transpose p, \, 0} + c_2 \max\braces{ -w \transpose q,\, 0},
	\end{align}
	for some constants $c_1, c_2 > 0$ and vectors $p$ and $q$ 
	that we will choose later. 
	Specifically, $p$ and $q$ are positioned such that the minimizer of 
	$\loss_{\text{perturbed}}$ is identical to that of $\loss$.
	Let $\braces{\tilde w_t}_t$ be the trajectory of 
	running GD on $\loss_{\text{perturbed}}$ with the step size 
	$\eta = \gamma/\lambda_{\max}(\nabla^2\loss_{\text{perturbed}}(w^*)) = \eta$,
	and the same initialization $\tilde w_0 = w_0$.
	
	Observe that the two ReLUs are only activated for all points above the 
	line $w\transpose p=0$ and below the line $w\transpose q =0$. We 
	design our original trajectory such that neither gets activated until 
	$w_{k-2}$, thus the two trajectories coincide for all $w_t$ from 
	$t=0$ to $t=k-2.$ The point $\tilde w_{k-2} = w_{k-2}$
	crosses over the line defined by $w\transpose p=0$, thus 
	activate the first ReLU, kicking the perturbed trajectory 
	in the direction orthogonal to $\tilde w_{k-2}$
	(as in the middle plot of \cref{fig:perturbed-loss-2d}).
	At this new point $\tilde w_{k-1}$, 
	we carefully set $c_1$, $c_2$, $p$ and $q$ such that 
	the first ReLU becomes inactive while the second activates, and that 
	another GD step takes us back to $w_0$. 
	\begin{figure}
		\centering \includegraphics[width=\linewidth]{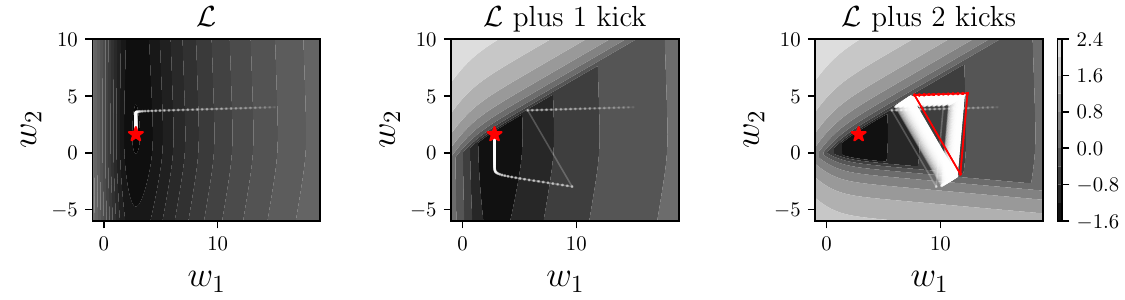}
		\vspace*{-1.2em}
		\caption{
			Illustration of our cycle construction in two dimensions using the logistic loss. 
			The objective (in log scale) in the middle and right contours are created by adding 
			large-magnitude examples to the dataset underlying $\loss$, 
			depicted on the left. Red stars mark the minimizer $w^*$ 
			of the respective losses. 
			Starting at the same $w_0$ 
			using the same $\gamma=0.4$, GD on $\loss$ and $\loss$ with 
			one set of additional data points both converge to 
			the solution. When we add in the second set of data points, 
			GD converges to a period-$13$ cycle.
			The trajectory in white marks the first $500$ steps, while the final 
			$20$ steps are in red.
			Details on the synthetic dataset are in \cref{sec:exps}. A discussion 
			on the basin of attraction of such limit cycle is in
			\labelcref{sec:initial-conditions}.
			}
		\label{fig:perturbed-loss-2d}
	\end{figure}
	
	As in the one-dimensional case (\cref{thm:1d-convergence-to-cycle}), 
	we then use the fact that 
	$\loss_{\text{perturbed}}$ is the limit of a continuous perturbation 
	for which we can find a valid classification dataset.
	Stability of the resulting cycle is also straightforward.
\end{sproof}
While the proof technique is similar to what we used 
in the one-dimensional case, the two-dimensional case is actually a lot more 
interesting. If $\gamma$ were to remain in the $(1,2]$ range, 
then we could have trivially stacked two (or more) 
one-dimensional constructions to obtain a $d>1$ counterexample.
But because we are pushing $\gamma$ below $1$, the unperturbed trajectory 
on the base dataset cannot cross over $0$ in either dimensions as 
it did in the one-dimensional construction. 
The difficulty lies in positioning the two 
kicks albeit the existence of an invariant set in the base dataset.
\begin{remark}
	It is worth noting that by this construction, it appears that 
	we cannot obtain a cycle of period exactly $2$ as the perturbation 
	involves $2$ kicks, and so the length of the cycle is at least $3$. 
	Nonetheless, one can easily construct a period-$2$ cycle by simply raising 
	$\gamma$ to be between $1$ and $2$, 
	so that the trajectory can cross over $0$ in the first coordinate,
	to the left of which we can add just one ReLU to kick us
	back to $w_0$ in just one step. 
	As our goal is to show that in the two dimensional case, GD can converge to 
	cycles on certain datasets when $\gamma \leq 2$, what we have proven is 
	already sufficient. Although our result is only for the two-dimensional 
	setting, it can easily be generalized to any dimension $d>1$.
\end{remark}
\begin{cor}
	\label{cor:lift-from-2d}
	For any $\ell$ for which \cref{assmpt:individual-loss} holds,
	any $\gamma\in (0,2]$ and any $d > 2$, 
	there exists a $d$-dimensional 
	non-separable classification 
	problem of the form \labelcref{eq:classification-obj}, on which 
	there exists a GD trajectory under the step size 
	$\eta = \gamma/\lambda$
	that converges to a cycle of period $k>1$.
\end{cor}

\begin{proof}
	The easiest way to construct a counterexample in $d>2$ is to stack multiple 
	lower-dimensional counterexamples. 
	Let such a dataset be $(X',y')$ with $X'$ an $n'\times d'$ matrix
	and $d'=1$ or $2$. Let its solution be $\bar w'$, and the corresponding
	period-$k$ orbit under some $\gamma<2$ be 
	denoted as  $\bar w_1$, $\bar w_2$, ..., $\bar w_k$.
	Now consider a new classification problem with features
	$ X = I_k\otimes X'$, 
	where $I_k$ is the $k\times k$ identity matrix, and $\otimes$ 
	is the Kronecker product. 
	The corresponding label vector $y$ is just $y'$ repeated $k$ times.
	
	Essentially, we have stacked $k$ independent problems of 
	the original dataset, one across every $d'$ dimension(s). 
	As a result, the solution $\bar w$ is simply $\bar w'$ 
	repeated $k$ times. Note that both the gradient an the Hessian 
	are identical to that of the original problem in each dimension, 
	albeit scaled down by a factor of $k$ due to the increase in the number 
	of examples. The Hessian is also (block) diagonal 
	everywhere, so the new step size $\gamma/\lambda$ is the same 
	as the original but scaled up by $k$.
	This allows the extra $k$ to be cancelled out in the GD update.

	Now we can run GD with that step size initialized at $w_0\in\R^{kd'}$, and 
	as long as $w_0$ contains an entry (or two consecutive entries starting 
	at an even index) close enough to a point on the $k$-periodic orbit, 
	the trajectory will converge to a cycle. 
	In fact, one can mix and match any combination of low-dimensional 
	cycles like so. Choosing the appropriate $\gamma$ will guarantee 
	cycling in at least one of the dimensions. One can even embed 
	a low-dimensional counterexample in an arbitrary problem (of smaller 
	$\lambda$) in the same manner, potentially involving real data. 
\end{proof}

An interesting observation on the example in the previous proof is that 
it gives rise to dynamics that do not exactly match what often appears
in the EoS phenomenon \citep{cohen2022gradient}. 
For simplicity, let $d'=1$. Suppose we pick a specific 
initialization $w_0$ in the neighborhood of 
$(\bar w_1, \bar w_2, \dots, \bar w_k)$.
Then in each dimension, GD will converge to the original period-$k$ 
cycle, but off-sync by precisely one time step. 
That is, each point on this $d=k$-dimensional cycle 
will be a cyclic permutation of the previous point. 
Observe that in the original one-dimensional cycle, there is a unique point 
that attains the highest second derivative, 
one that is potentially greater than $\lambda=\gamma/\eta$.
Combining the fact that the Hessian is diagonal and the $d$-dimensional 
cyclic orbit is always a permutation, the Hessian will always have its maximum 
eigenvalue exactly equal to that of the peak, and this peak 
can be greater than $2/\eta$. As illustrated in \cref{fig:cyclic-permutation}, 
we now have an example where the step size used 
is below the stability threshold, the sharpness converges to a constant 
(albeit strictly above $2/\eta$), 
and the loss does not continue to decrease. This gives an example not covered by 
the explanation of \citet{damian2022self} on EoS, which predicts that GD often 
self-stabilizes once the sharpness increases above the $2/\eta$.
However, their prediction relies on the existence of progressive sharpening, 
the lack of which in our example might be the reason 
behind this non-self-stabilizing behaviour. 

\section{Discussion}
\label{sec:conclusion}

We consider GD dynamics on classification tasks where the objective 
shares similar properties as that of logistic regression. 
Specifically, we study convergence behaviour when the problem is 
not linearly separable, and show that as the step size increases, 
GD follows a period-doubling route to chaos. We prove that although 
$\eta<2/\lambda$ is a necessary condition for global convergence, it is not 
sufficient, in that GD can converge to stable cycles below this step size 
on certain datasets. One limitation of our work is that 
beyond the one-dimensional setting, we still do not 
have a sufficient condition above the $2/L$ step size for which 
global convergence is guaranteed. Additionally, our analysis does not 
provide any practical recommendation --- even in one-dimension when 
the sufficient $1/\lambda$ step size is potentially greater than $2/L$, it is 
still impossible to know what $\lambda$ is without first finding the solution. 
Nevertheless, we hope our study provides insights into analyzing GD convergence 
with large step sizes on more general tasks involving the logistic 
or the cross-entropy loss.

Although the results are mostly negative on the global convergence of 
GD above the $2/L$ step size, 
the way we construct our cycles already contains an 
interesting observation: it seems that the datasets we 
construct to generate a cycle all require a few massive ``outliers''. 
These outliers are responsible for the large gradients on 
one side of the objective that kick our trajectory back to the starting location.
Indeed, the real datasets that we used 
to generate \cref{fig:real-main} both have features scaled within the $[-1,1]$ 
range. This scaling is perhaps why it appears that all runs 
below the critical $2/\eta$ step size actually
converged to the minimizer, despite using a large scale for initializations.
The reason we chose scaled data is to improve the conditioning so that 
GD can converge within a reasonable number of steps under tiny step sizes, 
so that generating the bifurcation diagrams can be done in a reasonable amount of time.
This naturally raises some interesting questions. 
Does data normalization or feature scaling (i.e. requiring $\| x_i \| = 1$)
allow GD to converge with even larger step sizes, potentially 
up to $2/\lambda$? Does it explain why techniques such as 
layer normalization \citep{ba2016layer} in deep neural networks help stabilize training?
We leave these questions to future work.

\vskip 0.2in
\acks{
We would like to thank Frederik Kunstner, Aaron Mishkin, Liwei Jiang, and 
Sharan Vaswani for helpful discussions and feedback on the manuscript. 
S.M. was partially supported by the NSERC PGS-D award 
(PGSD3-547276-2020). 
A.O. acknowledges the financial support of the Hector Foundation.
C.D. was supported by the NSF CAREER Award (2046760). 
}

\vskip 0.2in
\bibliography{refs}

\newpage
\appendix
%%% TOC %%%
\renewcommand{\contentsname}{Appendix}
\tableofcontents
\clearpage
%%% TOC %%%

\addtocontents{toc}{\protect\setcounter{tocdepth}{3}}

\section{Proofs in one dimension}
\label{sec:one-d-proofs}
For all proofs in this section, 
we assume without loss of generality that $y_i=1$ for all $i$. 
If $y_i=-1$, we can simply flip the corresponding $x_i$ to $-x_i$ 
to obtain the same objective. 
We use $T$ to denote the GD map with a constant step size $\eta$
\begin{align}
	T(w) \coloneqq w - \eta \loss'(w).
\end{align}

Recall our \cref{assmpt:individual-loss} which holds for the logistic loss.
\lossassumptions*

\subsection{Convergence under the stable step size}
\label{sec:1d-convergence}
\onedimconvergence*

\begin{proof}
	We assume without loss of generality that $w^*>0$. 
	If $w^*=0$, then fact that $\loss''$ is maximized at $0$ implies
	$T$ is Lipschitz with a constant strictly 
	less than $1$ (because $\loss''(w^*)> 0$ when the data is non-separable). 
	Therefore $T$ is a contraction 
	which implies global convergence. If $w^*<0$, we can flip the signs of 
	all the $x_i$'s and get $w^*>0$ as the problem is symmetric 
	about $0$. The proof is split into three cases based on $w_0$:	

	\textbf{Case 1}: $w_0 > w^*$. 
	By \cref{assmpt:individual-loss}, 
	$\ell''$ is decreasing on $[0,\infty)$, so
	$\ell'$ is concave on this interval, implying
	\begin{align}
		\label{eq:concavity-of-lprime}
		\loss'(w) &\leq \loss'(w^*) + \loss''(w^*)(w-w^*) = \loss''(w^*)(w-w^*).
	\end{align}
	Taking one step from some $w > w^*$ must satisfy 
	\begin{align*}
		T(w) = w - \eta \loss'(w) < w
	\end{align*}
	since $\loss'(w)>0$ for all $w>w^*$, and 
	\begin{align*}
		T(w) \geq w - \eta \paren{\loss''(w^*)(w-w^*)} \geq w^*
	\end{align*}
	using the upper bound on the step size. Convergence is thus guaranteed as 
	$\gI= [w^*, +\infty)$
	is an invariant set. Furthermore, $\loss''$ is maximized at $w^*$ over $\gI$, 
	and is $\loss''(w_0)$-strongly convex over $[w^*, w_0]$.
	Then for any $t>0$ starting at $w_0$, the classical 
	smooth strongly-convex analysis 
	(for instance \citet{garrigos2023handbook}) 
	guarantees for all $\eta \leq 1/\loss''(w^*)$,
	\begin{align} 
		\label{eq:rate-in-I}
		(w_t-w^*)^2 &\leq \paren{ 1- \eta\loss''(w_0) }^t (w_0 - w^*)^2.
	\end{align}
	% for all $\eta \leq 1/\loss''(w^*)$.

	For all $w_0<w^*$, we either have $w_0<T(w_0)<w^*$ or $T(w_0)\geq w^*$,
	the latter means we have crossed over into $\gI$ within which  
	convergence is already established. 
	When we do cross over at some $w < w^*$, the farthest we can land on 
	the right is given by 
	\begin{align}
		\label{eq:farthest-we-can-go}
		T(w) &\leq w^* - \eta \loss'(w) \nonumber \\
		&\leq w^* - \eta \lim_{w\to -\infty} \loss'(w) \nonumber \\
		&= w^* + \eta C < \infty
	\end{align}
	where $C\coloneqq \lim_{w\to -\infty } |\loss'(w)|$ is finite.

	\textbf{Case 2}: $w_0<0$. 
	For any $w_t<0$, convexity of $\loss$ and $w^*>0$ imply 
	$\loss'(w_t) \leq \loss'(0) < 0$ (since $\loss'$ is increasing 
	and $\loss'(w^*)=0$). This gives us
	\begin{align*}
		w_{t+1} - w_t = -\eta \loss'(w_t) = \eta\abs{\loss'(w_t)} \geq \eta \abs{\loss'(0)} > 0.
	\end{align*}	
	So the the maximum number of iterations to arrive at 
	$w_\tau \geq 0$ is 
	\begin{align}
		\label{eq:max-iters-to-enter-positive}
		\tau = \ceil{ \frac{-w_0}{\eta \cdot \abs{\loss'(0)}} }.
	\end{align}

	\textbf{Case 3}: $0 \leq w_0 \leq w^*$. In this interval, 
	\begin{align*}
		T(w_0) = w_0 - \eta \loss'(w_0) \geq w_0 - \eta \loss''(w^*)(w_0-w^*) = w_0 - (w_0-w^*) = w^*,
	\end{align*}
	where the first inequality uses \labelcref{eq:concavity-of-lprime}
	and the last equality uses the step size $\eta = 1/\loss''(w^*)$.
	So $\tau=1$ as we cross over into $\gI$ in one step. 
	
	Altogether \labelcref{eq:rate-in-I,eq:max-iters-to-enter-positive,eq:farthest-we-can-go} 
	implies that after at most 
	$1 + \max\braces{0, \, \ceil{-w_0\cdot \loss''(w^*)/\abs{\loss'(0)}}}$
	number of iterations, we will enter the invariant set $\gI=[w^*,+\infty)$, 
	from which convergence is guaranteed 
	at a linear rate $\rho$. The rate depends on
	$\eta = 1/\loss''(w^*)$ as well as  $w_\tau$ if we did not initialize in 
	$\gI$. If $w^* < 0$, then the initial number of iterations kick in when 
	$w_0 > 0$ and the entire argument holds with a sign flip. 
	This completes the proof.
\end{proof}

\subsection{Cycle construction below the critical step size}
\label{sec:1d-cycle}

\onedimcycle*
\begin{proof}
We first construct for all $\gamma\in[1,2)$ a dataset for which the 
objective is given by 
\begin{align*}
	\loss(w) = \sum_{i=1}^{n+m} \ell\paren{-y_ix_iw},
\end{align*}
in which all labels $y_i=1$, 
and there are $n$ copies of $x_i=-1$, and $m>n$ copies of $x_i=1$, 
such that $w^* > 0$. We drop the scaling factor $1/(n+m)$ for simplicity.
By \cref{lemma:cross-over}, there exist positive integers $m>n$ 
and a point $\bar w > w^*$ such that $T(\bar w) < 0$ under the step size 
$\eta = \gamma/\loss''(w^*)$. Set $w_{k-2}\coloneqq \bar w$ 
and $w_{k-1}\coloneqq T(\bar w) < 0$, so the trajectory crosses $0$ 
from $w_{k-2}$ to $w_{k-1}$. 
By \cref{lemma:extending-gd-trajectory} applied repeatedly 
for each $w_t > w^*$ there exists $w_{t-1} > w_t$ with 
$T(w_{t-1}) = w_t$. Extending backwards from $w_{k-2}$ we obtain the 
trajectory 
\begin{align}
	\label{eq:original-trajectory}
	w_{k-1} < 0 < w_{k-2} < \dots < w_1 < w_0
\end{align}
until we reach $w_0$, such that $w_k = T(w_{k-1}) < w_0$.

Now let's define the perturbation
\begin{align}
	\label{eq:ell-epsilon-to-add}
	\ell_\epsilon(w) \coloneqq \begin{cases}
		\epsilon^2\cdot\ell\paren{ - \frac{w}{\epsilon^2}} & \text{if $\epsilon \neq 0$} \\
		\max\{-w,\, 0\} & \text{if $\epsilon = 0$}
	\end{cases}
\end{align}
and the constant
\begin{align}
	\label{eq:defn-c}
	c \coloneqq \frac{w_0 - T(w_{k-1})}{\eta} > 0,
\end{align}
where $\eta$ is the step size 
that generated the trajectory in \labelcref{eq:original-trajectory}.
Then for all $\epsilon$ such that $c\cdot \epsilon^2=p/q>0$, 
the perturbed objective
\begin{align}
	\label{eq:perturbed-loss}
	\loss_\epsilon(w) \coloneqq \loss(w) + c\cdot \ell_\epsilon(w)
\end{align}
is equivalent to a classification problem on a dataset consisting 
of $q$ copies of the original examples plus $p$ copies of 
the example $x=1/\epsilon^2$. 
By \cref{assmpt:individual-loss},
\begin{align*}
	\lim_{\epsilon\to 0} \epsilon^2 \cdot \ell\paren{-\frac{w}{\epsilon^2}} = \max \braces{-w, \, 0}
\end{align*}
and so $\ell_\epsilon$ is continuous.
Let $w^*_\epsilon > 0$ 
be the minimizer of $\loss_\epsilon$ for $\epsilon>0$.
The GD map with the step size
$\eta_\epsilon = \gamma/\loss_\epsilon''(w^*_\epsilon)$ on
\labelcref{eq:perturbed-loss} using the same $\gamma$ 
that generated \labelcref{eq:original-trajectory} is given by
\begin{align}
	\label{eq:perturbed-gd-map}
	T_\epsilon(\tilde w) \coloneqq \tilde w - \eta_\epsilon \loss_\epsilon'(\tilde w)
\end{align}
be the corresponding GD map. 
Starting at $\tilde w_0 = w_0$ in \labelcref{eq:original-trajectory}, 
let $\tilde w_{t+1} = T_\epsilon(\tilde w_t)$ for $t=0,1\dots, k$.
Since $w^*_\epsilon$, $\loss'_\epsilon$, and $\loss''_\epsilon$ are 
all continuous in $\epsilon$ (by \cref{lemma:ell-eps-cont-diff,lemma:continuous-perturbation-wstar-epsilon}), 
so is $\eta_\epsilon$ and $T_\epsilon$, with respective limits
\begin{align*}
\eta_\epsilon\to\eta \quad \text{and} \quad T_\epsilon \to T_0	
\end{align*}
as $\epsilon\to 0$.
Furthermore, for all $w>0$, $\lim_{\epsilon\to 0} \ell'_\epsilon(w)  = 0$.
Combined with the original trajectory \labelcref{eq:original-trajectory} 
being positive up to $w_{k-2}$, we also have $\tilde w_{t} \to w_t$ 
as $\epsilon\to 0$ for all $t=0,1,\dots, k-1$. We will show that the next step 
takes us back to $w_0$, resulting in a cycle.

Let $G:\R^2\to\R$ be defined as
\begin{align}
	G(\epsilon, w) \coloneqq T^{k}_\epsilon(w) - w = \underbrace{T_\epsilon \circ \hdots \circ T_\epsilon}_{k \text{ times}} (w) - w.
\end{align}
By \cref{lemma:perturbed-T-differentiable}, 
$T_\epsilon$ is continuously differentiable in $w$ and $\epsilon$ except at 
$w=0$ and $\epsilon=0$, and so is $T^k_\epsilon$ as it is a function composition via 
the chain rule.
Consider the point $(\epsilon, w) = (0, w_0)$. 
Note that we must have $w_0 > w^* > 0$ due to the trajectory construction 
\labelcref{eq:original-trajectory}.
Evaluating $G$ at this point gives us
\begin{align*}
	G(0, w_0) &= T_0(T^{k-1}_0( w_0)) - w_0 \\
	&= T_0 \paren{ w_{k-1} } - w_0 \\
	&= w_{k-1}- \eta L'(w_{k-1}) + \eta \cdot c  - w_0 \tag{$w_{k-1} < 0$}\\
	&= T(w_{k-1}) + w_0 - T(w_{k-1})  - w_0 \tag{Definition of $c$ in \labelcref{eq:defn-c}}\\
	&= 0,
\end{align*}
resulting in a cycle. In addition, at any $w\neq 0$, 
\begin{align*}
	T'_\epsilon(w) &= 1 - \eta_\epsilon \loss''_\epsilon(w), \\
	\frac{\partial G}{\partial w}(0, w_0) &= \prod_{t=0}^{k-1} \paren{ 1 - \eta_0 \loss_0 ''( \tilde w_t ) } - 1.
\end{align*}
For all $t=0,1,\dots,k-2$, $\tilde w_t = w_t > w^*$ when 
$\epsilon=0$ and that $w^* = w^*_0$, so
\begin{align*}
	\loss_0''(\tilde w_t) < \loss_0''(w^*) = \gamma/\eta_0 \implies \abs{1-\eta_0\loss_0''(\tilde w_t)}<\abs{1-\gamma} \leq 1.
\end{align*}
For the last factor in the product ($t=k-1$), note that 
$\tilde w_{k-1} = w_{k-1} < 0$, so $\loss_0''(\tilde w_{k-1}) \leq \loss_0''(0)$.
The magnitude of this factor is $\abs{1-\eta_0\loss''(\tilde w_{k-1})}$ 
which is bounded. If needed, the trajectory can be extended to the 
right via repeated applications of \cref{lemma:extending-gd-trajectory} 
(thereby increasing $w_0$ and hence $k$), the product of the first $k-1$ 
factors can be made arbitrarily small (each has magnitude less than 
$\braces{1-\gamma}$
and there are $k-1$ of them), so the overall product is eventually 
less than $1$ regardless of the bounded last factor. This implies 
\begin{align*}
	\frac{\partial G}{\partial w}(0, w_0) \neq 0.
\end{align*}
We have justified the assumptions of the implicit function theorem 
(\cref{thm:implicit-function-theorem}) to conclude that 
there exists a function $\omega: I\to J$ where $I$ is an open interval about 
$\epsilon=0$ and $J$ is an open interval about $w=w_0$ such that 
$G(\epsilon, \omega(\epsilon))  = 0$ for all $\epsilon\in I$.

Lastly, we need to show that cycle is stable. Consider the function 
$\mu:\R^2\to\R_+$ defined as 
\begin{align}
	\mu(\epsilon, w) &\coloneqq \abs{ (T_\epsilon^{k})'(w)} = \abs{ \prod_{t=0}^{k-1}( 1 - \eta_\epsilon \loss_\epsilon''(T^t_\epsilon(w)) ) },
\end{align}
which is continuous in $w$ due to continuity of $\ell''_\epsilon$ 
in $w$ for all $\epsilon$, as well as continuity of $T_\epsilon$
except at $w=0$ and $\epsilon=0$ \cref{lemma:perturbed-T-differentiable}.
By \cref{lemma:ell-eps-cont-diff}, $\ell_\epsilon''(w)$ 
is continuous almost everywhere, and so is 
$\loss''_\epsilon$. Combining \cref{lemma:continuous-perturbation-wstar-epsilon}
for continuity of $\eta_\epsilon$ and \cref{lemma:perturbed-T-differentiable}
for $T_\epsilon$, we have that $\mu$ is continuous in both variables 
except at $w=0$ and $\epsilon=0$.

We know that  $\mu(0, w_0) < 1$ from earlier, and by continuity there must 
exist a neighborhood $U$ around $\epsilon=0$ such that 
$\mu(\epsilon, \omega(\epsilon)) < 1$ for all $\epsilon\in U$, as $\omega$ is 
continuous in $\epsilon$. Since $U$ and $I$ are both open, there must exist 
a nonzero $\epsilon_0\in I \cap U$ such that
\begin{align*}
	G(\epsilon_0, \omega(\epsilon_0)) = 0 \quad \text{and} \quad \mu(\epsilon_0, \omega(\epsilon_0)) < 1,
\end{align*}
implying that the trajectory on the perturbed objective with $\epsilon = \epsilon_0$ 
is a stable cycle. Stability ensures that if initialize close enough 
to $\omega(\epsilon_0)$ we will converge to this cycle, as it is a 
stable fixed point of the $T^{k}$ map.

A final note is that we need to rescale the original objective 
with the added perturbation such that it corresponds to a 
finite-sum problem \labelcref{eq:classification-obj} with integer 
copies of examples. Since the set of rationals are dense in $\R$, 
so is the set $\{\epsilon \in I\cap U:\, c\cdot \epsilon^2 \in \mathbb{Q}\}$.
Then we can pick $\epsilon_0$ from this set, so that 
the perturbed objective \labelcref{eq:perturbed-loss} 
indeed corresponds to a valid classification problem. This completes the proof.
\end{proof}

\begin{lem}[Crossing over $0$]
	\label{lemma:cross-over}
	Suppose the $1$-dimensional classification objective has the form
	\begin{align}
		\label{eq:cross-over-obj}
		\loss(w) = \frac{1}{n+m} \sum_{i=1}^{n+m} \ell(-y_ix_i w),
	\end{align}
	where $y_i =1$ for all $i$, and there are $n$ copies of $x_i=-1$ and 
	$m>n$ copies of $x_i=1$. Let $w^*$ be the minimizer of $\loss$.
	Assume the individual loss function $\ell$ 
	satisfies \cref{assmpt:individual-loss}.
	Then for all $\gamma\in(1,2]$, there exists positive integers 
	$m>n$, such that with a step size of 
	$\eta = \gamma/\loss''(w^*)$, there exists a point $\bar w > w^* > 0$
	at which one GD step takes us to some $T(\bar w) < 0$.
	That is, we cross over from the right of $w^*$ to 
	a point below $0$.
\end{lem}
\begin{proof}
	The objective can be simplified as 
	\begin{align*}
		\loss(w) &= \frac{1}{n+m} ( n\ell(w) + m\ell(-w) ) .
	\end{align*}
	Up to a scaling factor, minimizing this is equivalent to minimizing 
	\begin{align}
		\loss_c(w) &= \ell(w) + (1+c^2)\ell(-w)
	\end{align}
	for some $c\in\R$, and $w^*_c$ be the minimizer of this objective.
	Let $T_c$ denote the GD map 
	\begin{align}
		\label{eq:eta-c}
		T_c(w) \coloneqq w - \eta_c \loss_c'(w) 
		\quad \text{where} \quad \eta_c \coloneqq \frac{\gamma}{\loss_c''(w^*_c)}.
	\end{align}
	Consider the case of $c=0$. As all examples $x_i$ have the same magnitude 
	with the same label, the minimizer is at $w^*=0$. Thus 
	the step size is simply $\eta = \frac{\gamma}{\loss''(0)}$.
	The derivative of the GD map in this case is just 
	\begin{align*}
		T'(w) &= 1 - \eta \loss''(w) = 1 - \frac{\gamma}{\loss''(0)} \loss''(w).
	\end{align*}
	Since $\ell''$ is decreasing on $[0, \infty)$, so is $\loss''$.
	By continuity, there must exist $w>0$ such that $T'(w) < 0$, as 
	$\gamma > 1$. 
	As $T(0) = 0$ and $T'(0) = 1 - \gamma < 0$, $T$ is decreasing in the 
	neighborhood of $0$ to the right. Thus there exists $\tilde w>0$ 
	depending on the value of $\gamma\in (1,2]$ sufficiently 
	close to $0$ such that $T(\tilde w) < 0$.
	
	Now define $G:\R^2\to\R$ to be
	\begin{align*}
		G(c, w) = T_c(w) - T(\tilde w)
	\end{align*}
	At the point $c=0$ and $w=\tilde w$, clearly $G(0, \tilde w) = 0$, and 
	\begin{align*}
		\frac{\partial G}{\partial w} (0, \tilde w) =  1 - \eta_0 \loss_0''(\tilde w) \neq 0.
	\end{align*}

	Define $\phi:\R^2\to\R$ to be $\phi(c,w) \coloneqq \loss_c(w)$. 
	Clearly, $\phi$ is twice-continuously differentiable in both both variables 
	as $\ell$ is twice-continuously differentiable. 
	\cref{assmpt:individual-loss} on $\ell$ lets us invoke 
	\cref{cor:continuity-of-argmin} to have that $w^*_c$ is continuous and 
	differentiable as a function of $c$. 
	Finally, by the chain rule, $T_c(w)$ is also continuously differentiable 
	as a function of both $c$ and $w$.
	Together, we have that $G$ 
	is continuously differentiable in a neighborhood of the point
	$c=0$ and $w=\tilde w$.
	By the implicit function theorem (\cref{thm:implicit-function-theorem}), 
	there exists a function $\omega:I\to J$ where $I$ is an open interval about 
	$c=0$ and $J$ is an open interval about $w=\tilde w$ such that 
	$G(c, \omega(c)) = 0 $ for all $c\in I$. It suffices to pick a nonzero
	and rational $c_0\in I$ such that $\bar w = \omega(c_0)> 0$ and 
	$G(c_0, \bar w) = T_{c_0}(\bar w) - T(\tilde w) = 0$.
	Therefore, we can rescale the objective $\loss_c$ such that 
	it is equivalent to \labelcref{eq:cross-over-obj} with a valid combination 
	of $m$ and $n$, on which one GD step with step size \labelcref{eq:eta-c}
	takes us from $\bar w >0$ to $T_c(\bar w) = T(\tilde w) < 0$.
\end{proof}

\begin{lem}[Trajectory extension]
	\label{lemma:extending-gd-trajectory}
	Suppose $d=1$ and the minimizer of the objective \labelcref{eq:classification-obj} 
	satisfies $w^*>0$, where \cref{assmpt:individual-loss} on the individual 
	loss function holds. 
	Then for all $w' > w^*$ and $\eta>0$, 
	there exists $w>w'$ such that $T(w) = w'$.
\end{lem}
\begin{proof}
	By the Lipschitz assumption on $\loss$, $\lim_{w\to +\infty} T(w) = +\infty$.
	Since $T(w^*)=w^*$ and $T$ is continuous, by the intermediate value theorem 
	applied to $T(w)-w'$ on $[w^*,\infty)$, there exists $w>w^*$ with 
	$T(w)=w'$ for any $w'>w^*$. 
\end{proof}

\begin{lem}[Continuous differentiability of perturbed loss]
	\label{lemma:ell-eps-cont-diff}
	Let $\rho:\R^2\to\R$ be defined as 
	\begin{align*}
		\rho(\epsilon, w) = \begin{cases}
			\epsilon^2 \ell\paren{-\frac{w}{\epsilon^2}}  & \text{if $\epsilon \neq 0$} \\
			\max\{-w,\, 0\} & \text{if $\epsilon = 0$.}
		\end{cases}
	\end{align*}
	where $\ell$ satisfies \cref{assmpt:individual-loss}.
	Then $\rho$ has continuous first partial derivatives 
	everywhere except at $w=0$ and $\epsilon=0$. 
	The second partial 
	derivatives $\frac{\partial^2\rho}{\partial^2}, \frac{\partial^2\rho}{\partial\epsilon\partial w}, \frac{\partial^2\rho}{\partial w\partial \epsilon}$
	are continuous everywhere except $w = 0$ and $\epsilon=0$, with
	$\frac{\partial^2\rho}{\partial^2\epsilon}$ continuous on $w>0$.
\end{lem}
\begin{proof}
	The partial derivative with respect to $w$ is 
	\begin{align}
		\frac{\partial \rho}{\partial w}(\epsilon, w) = \begin{cases}
				-\ell'\paren{-\frac{w}{\epsilon^2}} & \text{if $\epsilon \neq 0$} \\
				-1 & \text{if $\epsilon = 0$ and $w<0$}\\
				0 & \text{if $\epsilon = 0$ and $w>0$}
		\end{cases}
	\end{align}
	which is continuous everywhere except at $w=0$, since 
	$\lim_{\epsilon\to 0} -\ell'(-w/\epsilon^2) = -1 $ for $w < 0$ and 
	$0$ for $w>0$. 
	The partial derivative with respect to $\epsilon$ when $\epsilon\neq 0$ is
	\begin{align}
		\frac{\partial \rho}{\partial \epsilon}(\epsilon, w) = 2\epsilon\cdot\ell\paren{ - \frac{w}{\epsilon^2}} + \frac{2w}{\epsilon}\cdot\ell'\paren{ -\frac{w}{\epsilon^2} },
	\end{align}
	which is clearly continuous. For all $w\in\R$, 
	taking the limit as $\epsilon\to 0$, if $w \geq 0$, 
	then the first term goes to $0$ as 
	$\lim_{z\to-\infty} \ell(z) = 0$. For the second term, we have 
	\begin{align*}
		\lim_{\epsilon \to 0} \frac{2w}{\epsilon} \ell'\paren{-\frac{w}{\epsilon^2}} &= \lim_{\epsilon \to 0}  \frac{ 4w^2 }{ \epsilon^3 } \ell''\paren{-\frac{w}{\epsilon^2}}  = 0
	\end{align*}
	where we first apply L'Hôpital's rule 
	followed by using our assumption on the decay rate of $\ell''$.
	Similarly, if $w < 0$, we can apply L'Hôpital's rule on 
	this term and merge with the second to obtain 
	\begin{align*}
		\lim_{\epsilon \to 0}  \frac{\partial \rho}{ \partial \epsilon } (\epsilon, w) &= \lim_{\epsilon \to 0} -\frac{4w}{\epsilon} \ell'\paren{-\frac{w}{\epsilon^2}} + \frac{2w}{\epsilon}   \ell'\paren{-\frac{w}{\epsilon^2}} \\
		&= \lim_{\epsilon \to 0} -\frac{2w}{\epsilon} \ell'\paren{-\frac{w}{\epsilon^2}} \\
		&= \lim_{\epsilon \to 0} -\frac{ 4w^2 }{\epsilon^3} \ell''\paren{-\frac{w}{\epsilon^2}} \tag{ \text{L'Hôpital's rule again} }
	\end{align*}
	which is again $0$ using the decay rate assumption. 
	At $\epsilon=0$, 
	\begin{align*}
		\left. \frac{\partial \rho }{ \partial \epsilon} (\epsilon, w) \right\rvert_{\epsilon=0} &= \lim_{t\to 0} \frac{ 1}{t} \paren{ t^2\ell\paren{-\frac{w}{t^2}} - \max\braces{-w,\, 0}  } = 0
	\end{align*}
	as long as $w \neq 0$,
	which proves continuity of the partial derivative of $\rho$ with respect to $\epsilon$ 
	for all $w$. Therefore, $\rho$ has continuous partial derivatives in both 
	variables except at $w=0$ and $\epsilon=0$.
	For the second derivatives, consider taking derivatives again with respect to $w$,
	\begin{align}
		\frac{\partial^2 \rho }{\partial^2 w } (\epsilon, w) &= \begin{cases}
			\frac{1}{\epsilon^2}\ell''\paren{-\frac{w}{\epsilon^2}} & \text{if $\epsilon\neq 0$} \\
			0 & \text{if $\epsilon = 0$ and $w\neq 0$},
		\end{cases} \\
		\frac{\partial^2 \rho }{\partial w \partial \epsilon} (\epsilon, w) &= \begin{cases}
			-\frac{2w}{\epsilon^3} \ell''\paren{ -\frac{w}{\epsilon^2} } & \text{if $\epsilon \neq 0$} \\
			0 & \text{if $\epsilon = 0$}
		\end{cases} 
	\end{align}
	which is continuous in $w$ as long as $w$ and $\epsilon$ are not both $0$,
	and has limit $0$ as $\epsilon\to 0$ using our decay rate assumption,
	and thus also continuous in $\epsilon$.
	Furthermore, 
	\begin{align}
		\frac{\partial^2\rho}{\partial \epsilon \partial w}(\epsilon, w) &= \begin{cases}
			-\frac{2w}{\epsilon^3} \ell''\paren{-\frac{w}{\epsilon^2}} & \text{if $\epsilon\neq 0$} \\ 
			0 & \text{if $\epsilon=0$ and $w\neq 0$}
		\end{cases} \\ 
		\frac{\partial^2\rho}{\partial^2 \epsilon} (\epsilon, w) &= 2\ell\paren{-\frac{w}{\epsilon^2}} + \frac{2w}{\epsilon^2}\ell'\paren{-\frac{w}{\epsilon^2}} - \frac{2w}{\epsilon^2} \ell'\paren{-\frac{w}{\epsilon^2}} + \frac{2w^2}{\epsilon^4}\ell''\paren{-\frac{w}{\epsilon^2}} \nonumber \\
		&= 2\ell\paren{-\frac{w}{\epsilon^2}} + \frac{2w^2}{\epsilon^4}\ell''\paren{-\frac{w}{\epsilon^2}}  \nonumber 
	\end{align}
	which is continuous in $w$ and $\epsilon$ for all
	$(\epsilon, w)\neq (0,0)$.
	At the point $(\epsilon, w)=(0, 0)$, 
	\begin{align*}
		\lim_{\epsilon\to 0} \frac{\partial^2\rho}{\partial\epsilon\partial w}(\epsilon, w) = 0
	\end{align*}
	using our decay rate assumption, and 
	\begin{align*}
		\left. \frac{\partial^2 \rho }{ \partial \epsilon \partial w} (\epsilon, w) \right\rvert_{\epsilon=0} &= \lim_{t\to 0}  \frac{1}{t} \paren{ -\ell'\paren{-\frac{w}{t^2}} - \frac{\partial \rho}{\partial w} (0, w) } = 0
	\end{align*}
	as long as $w\neq 0$, and so $\frac{\partial^2 \rho }{ \partial \epsilon \partial w}$ 
	is continuous at $\epsilon=0$ as long as $w\neq 0$.
	Finally, for $w > 0$, 
	\begin{align*}
		\lim_{\epsilon \to 0} \frac{\partial^2\rho}{\partial^2\epsilon}(\epsilon, w) = 0
	\end{align*}
	using the limit of $\ell$ and the decay rate assumptions, and 
	\begin{align*}
		\left. \frac{\partial^2\rho}{\partial^2\epsilon}\right\rvert_{\epsilon=0} &= \lim_{t\to 0} \frac{1}{t} \paren{ 2t\ell\paren{-\frac{w}{\epsilon^2}} + \frac{2w}{\epsilon}\ell'\paren{-\frac{w}{t^2}} - \frac{\partial \rho}{\partial \epsilon}(0, w) } \\
		&= \lim_{t\to 0} \frac{1}{t} \paren{ 2t\ell\paren{-\frac{w}{t^2}} + \frac{2w}{t}\ell'\paren{-\frac{w}{t^2}} - 0} \\
		&= 0
	\end{align*}
	following previous arguments.

	In summary, we have shown that $\rho$ has continuous first partial derivatives
	everywhere except $w=0$ and $\epsilon=0$, and continuous second partial 
	derivatives $\frac{\partial^2\rho}{\partial^2 w}, \frac{\partial^2\rho}{\partial\epsilon\partial w}, \frac{\partial^2\rho}{\partial w\partial \epsilon}$
	everywhere except $w = 0$ and $\epsilon=0$, and lastly 
	$\frac{\partial^2\rho}{\partial^2\epsilon}$ is continuous on $w>0$.
\end{proof}

\begin{lem}[Continuous differentiability of perturbed minimizer]
	\label{lemma:continuous-perturbation-wstar-epsilon}
	Suppose \cref{assmpt:individual-loss} holds.
	As a function of $\epsilon$, the minimizer $w^*_\epsilon$ 
	of the perturbed objective \labelcref{eq:perturbed-loss}
	is continuous everywhere on $\R$ 
	and continuously differentiable everywhere except at $\epsilon=0$. 
\end{lem}
\begin{proof}
	We will invoke \cref{cor:continuity-of-argmin}
	to show that $w^*_\epsilon$ is a continuous function in $\epsilon$.
	To verify the assumptions, first observe that 
	as a function of $\epsilon\in\R$ and $w\in\R$, 
	the perturbed loss $\loss_\epsilon$ is clearly proper as it never attains $-\infty$
	and its (effective) domain is nonempty. It is also continuous in both 
	$\epsilon$ and $w$ by definition of $\ell_\epsilon$ and continuity of $\loss$.
	By \cref{assmpt:individual-loss}, $\ell$ is strictly convex, and so is 
	$\loss_\epsilon$ as it is the sum of a strictly convex function with
	a convex $\ell_\epsilon$.

	It remains to show that the horizon function (\cref{defn:horizon-functions})
	$\loss^\infty_0(w)$ is positive for all $w\neq 0$.
	Using the formula for a convex $\loss_\epsilon$ 
	from \cref{thm:horizon-functions},
	we have for all $w\neq 0$ and $\bar x \in\R$,
	\begin{align}
            \label{eq:verify-horizon-function-positive}
		\loss^\infty_0(w) &= \sup_{\tau\in(0,\infty)} \frac{ \loss_0( \bar x + \tau w ) - \loss_0(\bar x) }{\tau} >0
	\end{align}
	since the difference quotient for a convex function is non-decreasing on 
	$(\tau, \infty)$, and the infimum as $\tau \to 0$ is $0$ as long as $w\neq 0$.
	
	To show differentiability, let 
	\begin{align}
		G(\epsilon, w) = \loss_\epsilon'(w).
	\end{align}
	At point $(0, w^*)$, 
	$G(0, w^*) = \loss_0'(w^*) = \loss'(w^*) =0$. 
	In a neighborhood of this point, 
	$G$ is continuously differentiable due to $w^*>0$, 
	\cref{lemma:ell-eps-cont-diff} and 
	twice continuous-differentiability of $\loss$. 
	The implicit function theorem (\cref{thm:implicit-function-theorem}) thus
	guarantees the existence of a continuously differentiable function 
	$\omega: I\to J$ where $I$ is an open interval about $0$ and $J$ an 
	open interval about $w^*$ such that 
	$\omega(0) = w^*$ and $\loss'_\epsilon(\omega(\epsilon)) = 0$ for all $\epsilon\in I$.
	As $\loss_\epsilon$ is strictly convex, $\omega(\epsilon)$ must be its unique 
	minimizer. Setting $w^*_\epsilon$ to $\omega(\cdot)$ completes the proof.
\end{proof}

\begin{lem}[Continuous differentiability of perturbed GD map]
	\label{lemma:perturbed-T-differentiable}
	The function $\psi:\R^2\to\R$
	\begin{align*}
		\psi(\epsilon, w) = T_\epsilon(w) = w - \eta_\epsilon \loss'_\epsilon(w)
	\end{align*}
	is continuously differentiable everywhere except at 
	$(\epsilon, w) = (0, 0)$.
\end{lem}
\begin{proof}
	For $\epsilon \neq 0$, continuous differentiability of $\psi$ with respect to $w$ 
	follows from \cref{lemma:ell-eps-cont-diff} and that of 
	the original $\loss$. Additionally, recall that 
	\begin{align*}
		\eta_\epsilon = \frac{\gamma}{ \loss_\epsilon''(w^*_\epsilon) }.
	\end{align*}
	By \cref{lemma:continuous-perturbation-wstar-epsilon},
	$w^*_\epsilon$ is continuously differentiable.
	Together with the fact that $\loss'''$ is continuous, 
	we have $\eta_\epsilon$ must be continuously 
	differentiable in $\epsilon$. Therefore $\partial \psi/\partial \epsilon$
	must also be continuous. 
	If $\epsilon=0$, the result trivially holds in $w$ 
	by definition of $\loss_\epsilon$. 
\end{proof}

\section{Proofs in higher dimensions}
\label{sec:higher-d-proofs}

\subsection{Cycle construction in two dimensions}
\label{sec:cycle-construction-2d}

\twodimcycle*
\begin{proof}
	Consider the dataset constructed in \cref{lemma:cross-over-2d}
	where all labels are $y_i=1$, and the 
	$x_i\in\R^2$'s consist of $\m1$ copies of $e_1$, $\n1$ copies of $-e_1$, 
	$\m2$ copies of $e_2$, and $\n2$ copies of $-e_2$, where $e_i$ is the 
	$i$th coordinate vector. 
	The resulting loss, gradient, and Hessian are given by 
	\begin{align}
		\label{eq:2d-original-loss}
		\loss(w) &= \m1\ell(-w_1) + \n1\ell(w_1) + \m2\ell(-w_2) + \n2\ell(w_2), \\		
		\nabla \loss(w) &= \begin{pmatrix}
			\n1\ell'(w_1) - \m1\ell'(-w_1)  \\
			\n2\ell'(w_2) - \m2\ell'(-w_2) 
		\end{pmatrix}\\
		\nabla^2 \loss(w) &= \begin{pmatrix}
			\m1 \ell''(-w_1) + \n1\ell''(w_1) & 0 \\
			0 & \m2 \ell''(-w_2) + \n2\ell''(w_2)
		\end{pmatrix}.
	\end{align}
	The solution $w^*=(u^*,v^*)\transpose$ satisfies $u^*>0$ and $v^*>0$, and 
	the dataset is chosen without loss of generality 
	such that for $\gamma\in(0,1]$,
	\begin{align*}
		\eta = \frac{\gamma}{ \lambda_{\max}(\nabla^2\loss(w^*)) } = \frac{\gamma}{ \m1\ell''(-u^*) + \n1 \ell''(u^*) },
	\end{align*}
	that is, it only depends on the first coordinate of the solution.
	By \cref{lemma:cross-over-2d,lemma:extending-gd-trajectory-2d}, 
	under the step size $\eta$,
	there exists a trajectory where $\braces{w_t}_{t=0}^{k-3}$ all lie 
	in the positive quadrant below the line $ w\transpose p = 0$ 
	($w_t\transpose p< 0$),
	while $w_{k-2}$ crosses this line ($w_{k-2}\transpose p > 0$) 
	but remain to the right of $w^*$ in the positive quadrant. 
	Furthermore, next point $w_k$ will move 
	closer to $w^*$ in both directions without crossing, due to 
	the two dimensions being decoupled and that $\gamma\leq 1$. 
	Let $\w1^*$, $\w2^*$ be the two coordinates of $w^*$, and choose 
	\begin{align}
		p \coloneqq (\w2^*,\, -\w1^*)\transpose, \quad \text{and} \quad q \coloneqq w_0 - T(w_{k-1} + \eta \cdot c_1 \cdot p),
	\end{align}
	where $c_1>0$ will be specified later.
	Additionally, the two ``kicks'' will be activated by the perturbations
	\begin{align}
		\ell_{1,\epsilon}(w) = \begin{cases}
			\epsilon^2 \ell\paren{ -\frac{w\transpose p }{\epsilon^2} } & \text{if $\epsilon \neq 0$}  \\
			\max\braces{ -w\transpose p, \, 0} & \text{if $\epsilon=0$}
		\end{cases}  \qquad 
		\ell_{2,\epsilon}(w) = \begin{cases}
			\epsilon^2 \ell\paren{ -\frac{w\transpose q }{\epsilon^2} } & \text{if $\epsilon \neq 0$}  \\
			\max\braces{ -w\transpose q, \, 0} & \text{if $\epsilon=0$}
		\end{cases} .
	\end{align}
	Let $c_2 = 1/\eta$, 
	where $\eta$ is the step size that generated the 
	original trajectory described above.
	Then for all $\epsilon$ such that $c_1\cdot\epsilon^2$ and $c_2\cdot\epsilon^2$ are both rational, 
	the perturbed objective 
	\begin{align}
		\label{eq:perturbed-loss-2d}
		\loss_\epsilon(w) \coloneqq \loss(w) + c_1 \cdot \ell_{1,\epsilon}(w) + c_2 \cdot \ell_{2, \epsilon}(w)
	\end{align}
	is equivalent to a classification problem on a valid dataset. 
	\cref{assmpt:individual-loss} guarantees that both $\ell_{1,\epsilon}$ 
	and $\ell_{2,\epsilon}$ are continuous everywhere in $\epsilon$.

	\begin{minipage}{0.49\textwidth}
		Let $w_\epsilon^* = (u_\epsilon^*, v_\epsilon^*)\transpose$ be the minimizer 
		of $\loss_\epsilon$, for $\epsilon>0$. We can run GD with the step size 
		\begin{align}
			\eta_\epsilon = \frac{\gamma}{\lambda_{\max}(\nabla^2\loss_\epsilon(w^*_\epsilon))}
		\end{align}
		on this perturbed objective using the same $\gamma$ that generated 
		the original trajectory. Let 
		\begin{align}
			\label{eq:perturbed-gd-map-2d}
			T_\epsilon(\tilde w) \coloneqq \tilde w - \eta_\epsilon \nabla \loss_\epsilon(\tilde w)
		\end{align}
		be the corresponding GD map. 
		Let $\tilde w_{t+1} = T_\epsilon(\tilde w_t)$ be the GD iterates 
		under this perturbed map $T_\epsilon$, starting at $\tilde w_0 = w_0$.
		By continuity of $\eta_\epsilon$, $\nabla\loss_\epsilon$, and 
		$\nabla^2\loss_\epsilon$ (\cref{lemma:ell-eps-cont-diff-2d,lemma:continuous-perturbation-wstar-epsilon-2d}),
		$\eta_\epsilon \to \eta$ and $T_\epsilon \to T_0$
		as $\epsilon\to 0$.
		Observe that for all $t=0,1,\dots,k-3$, $w_t\transpose p > 0$ as all of 
		these iterates lie below the line connecting the origin and $w^*$ (the line 
		$w\transpose p = 0$). 
		And so in the limit as $\epsilon\to 0$, these points do not activate the 
		gradient in $\max\braces{-w\transpose p,\, 0}$. 
		Furthermore, all of these points lie above the $w\transpose q=0$ line
		by construction and do not activate the gradient in the second ReLU.
	\end{minipage}
	\hfill
	\begin{minipage}{0.48\textwidth}
		\centering\includegraphics[width=0.98\linewidth]{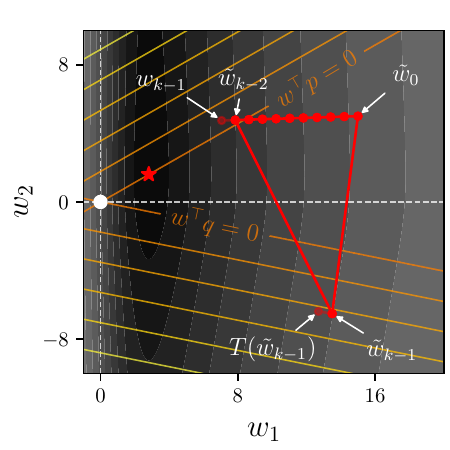}
		% \vspace{-4em}
		\captionof{figure}{Illustration of the two-dimensional cycle construction. The 
			grey contour corresponds to the level sets of 
			$\loss$, the original loss. The minimizer is marked with the star. 
			The two sets of straight contour 
			lines correspond to the two ReLUs that we add. Points on the cycle 
			are marked with solid red dots. }
		\label{fig:2d-proof-illustration}
	\end{minipage}

	At $w_{k-2}$, however, the first gradient becomes active, and thus 
	\begin{align*}
		\lim_{\epsilon\to 0} \tilde w_{k-1} = \lim_{\epsilon\to 0} T_\epsilon(\tilde w_{k-2}) &= \lim_{\epsilon\to 0} \tilde w_{k-2} - \eta_\epsilon \nabla \loss_\epsilon(\tilde w_{k-2}) \\
		&= w_{k-2} - \eta \nabla \loss(w_{k-2}) + \eta \cdot c_1 \cdot p +  0 \\
		&= w_{k-1} + \eta \cdot c_1 \cdot p.
	\end{align*}
	Note that we can set $c_1$ arbitrarily large to guarantee that 
	the second coordinate of this point is negative. Having $c_1$ fixed, 
	observe that 
	\begin{align*}
		(w_{k-1} + \eta \cdot c_1 \cdot p)\transpose q = (w_{k-1} + \eta \cdot c_1 \cdot p)\transpose (w_0 - T(w_{k-1} + \eta \cdot c_1 \cdot p)).
	\end{align*}
	Since $T(w_{k-1}+\eta\cdot c_1 p)$ is bounded while $w_0$ can be made 
	arbitrarily large by repeated backward extension via 
	\cref{lemma:extending-gd-trajectory-2d}, the $w_0$ term in the 
	expanded inner product dominates. Because $w_{k-1}+\eta c_1 p$ has a 
	negative second coordinate (from the first kick) while $w_0$ has a large 
	positive second coordinate, the inner product is eventually negative, 
	establishing that this overall inner product can be made negative.
	
	This allows us to activate the second 
	ReLU. Combining these, we have at the next point 
	\begin{align*}
		\lim_{\epsilon\to 0} T_\epsilon(\tilde w_{k-1}) &= \lim_{\epsilon\to 0} \tilde w_{k-1} - \eta _\epsilon \nabla \loss_\epsilon(\tilde w_{k-1}) \\
		&= w_{k-1} + \eta \cdot c_1 \cdot p - \eta \nabla \loss( w_{k-1} + \eta \cdot c_1 \cdot p ) - 0 + \eta \cdot c_2 \cdot q \\
		&= T(w_{k-1} + \eta \cdot c_1 \cdot p) + \eta \frac{1}{\eta} q \\
		&= T(w_{k-1} + \eta \cdot c_1 \cdot p) + w_0 - T(w_{k-1} + \eta \cdot c_1 \cdot p) \\
		&= w_0,
	\end{align*}
	giving us a cycle.
	Let $G:\R\times \R^2\to\R$ be defined as
	\begin{align}
		G(\epsilon, w) \coloneqq T^{k}_\epsilon(w) - w = \underbrace{T_\epsilon \circ \hdots \circ T_\epsilon}_{k \text{ times}} (w) - w.
	\end{align}
	By \cref{lemma:perturbed-T-differentiable-2d}$, 
	T_\epsilon$ is continuously differentiable in $w$ and $\epsilon$ except at 
	$\epsilon=0$ and $\braces{w:w\transpose p >0}$, and so is $T^k_\epsilon$ 
	as it is a function composition via the chain rule.
	Consider the point $(\epsilon, w) = (0, w_0)$. 
	Due to the trajectory construction, we would not have $w_0$ lie exactly on the 
	line $w\transpose p = 0$. 
	Evaluating $G$ at this point gives us
	\begin{align*}
		G(0, w_0) &= T_0(T^{k-1}_0( w_0)) - w_0 = 0
	\end{align*}
	as we saw that $\lim_{\epsilon \to 0} T_\epsilon(\tilde w_k) = w_0$.
	In addition, at any $w\neq 0$, the Jacobian of the GD map with respect 
	to $w$ is given by 
	\begin{align*}
		J_w T_\epsilon(w) = I - \eta_\epsilon \nabla^2\loss_\epsilon(w).
	\end{align*}
	By the chain rule, 
	\begin{align*}
		J_w G(\epsilon, w) &= J_wT_\epsilon^{k}(w) - I \\
		&= \prod_{t=0}^{k-1} \paren{ I - \eta_\epsilon \nabla^2\loss_\epsilon( T_\epsilon^t(w) ) } - I. 
	\end{align*}
	By \cref{assmpt:individual-loss}, $\ell''$ is decreasing on $[0, \infty)$.
	At $\epsilon=0$, $\nabla^2 \loss_0$  evaluated at every point on the 
	cycle has its largest eigenvalue strictly less than that at $w^*$
	due to the decoupling of the base loss (diagonal Hessian), 
	except possibly at $w_k$. Recall that we can make $w_k$'s second 
	coordinate arbitrarily negative by increasing $c_1$, yet still be 
	able to activate the second ReLU by extending $w_0$ backwards. 
	Thus we can guarantee that for both $i=1,2$
	\begin{align*}
		\abs{ \lambda_i \paren{ J_wT_0^{k}(w_0) } } < (1 - \gamma )^k < 1.
	\end{align*}
	That is, all eigenvalues of the Jacobian of $T^{k}$ lies within the 
	unit circle, when $\epsilon=0$ and $w=w_0$.
	This suffices to ensure that
	\begin{align*}
		\abs{ \lambda_i \paren{  J_w G(0, w_0) }  } \neq 0.
	\end{align*}
	
	We have now justified the assumptions of the implicit function theorem 
	(\cref{thm:implicit-function-theorem}) to conclude that 
	there exists a function $\omega: I\to J$ where $I$ is an open interval about 
	$\epsilon=0$ and $J$ is an open neighborhood about $w=w_0$ such that 
	$G(\epsilon, \omega(\epsilon))  = 0$ for all $\epsilon\in I$.
	
	Local stability of the cycle follows the same argument as in 
	the proof of \cref{thm:1d-convergence-to-cycle}, where 
	we define a perturbation $\mu:\R\times \R^2\to \R_+$ 
	\begin{align}
		\mu(\epsilon, w) &\coloneqq \abs{ (J_wT_\epsilon^{k})(w)}
	\end{align}
	which is continuous in $w$ due to continuity of $\ell''_\epsilon$ 
	in $w$ for all $\epsilon$, as well as continuity of $T_\epsilon$
	except on $\braces{w\transpose p = 0}$ and $\epsilon=0$ 
	\cref{lemma:perturbed-T-differentiable-2d}.
	By \cref{lemma:ell-eps-cont-diff-2d}, $\ell_\epsilon''(w)$ 
	is continuous almost everywhere, and so is 
	$\loss''_\epsilon$. Combining \cref{lemma:continuous-perturbation-wstar-epsilon-2d}
	for continuity of $\eta_\epsilon$ and \cref{lemma:perturbed-T-differentiable-2d}
	for $T_\epsilon$, we have that $\mu$ is continuous in both variables 
	except on $\braces{w\transpose p = 0}$ and $\epsilon=0$.
	
	We know that  $\mu(0, w_0) < 1$ from earlier, and by continuity there must 
	exist a neighborhood $U$ around $\epsilon=0$ such that 
	$\mu(\epsilon, \omega(\epsilon)) < 1$ for all $\epsilon\in U$, as $\omega$ is 
	continuous in $\epsilon$. Since $U$ and $I$ are both open, there must exist 
	a nonzero $\epsilon_0\in I \cap U$ such that
	\begin{align*}
		G(\epsilon_0, \omega(\epsilon_0)) = 0 \quad \text{and} \quad \mu(\epsilon_0, \omega(\epsilon_0)) < 1,
	\end{align*}
	implying that the trajectory on the perturbed objective with $\epsilon = \epsilon_0$ 
	is a stable cycle. Stability ensures that if initialized close enough 
	to $\omega(\epsilon_0)$ we will converge to this cycle, as it is a 
	stable fixed point of the $T^{k}$ map.
	A final note is that we can always pick $\epsilon_0$ such that 
	$c_1\cdot \epsilon_0^2$ and $c_2\cdot \epsilon_0^2$ are both rational, 
	so that the perturbed objective \labelcref{eq:perturbed-loss-2d} 
	indeed corresponds to a valid classification problem. 
	This completes the proof.
\end{proof}

\begin{lem}[Crossing over in 2D]
	\label{lemma:cross-over-2d}
	Let $N = \m1 + \m2 + \n1 + \n2$ for positive integers 
	$\m1, \m2, \n1,$ and $\n2$.
	Suppose the $2$-dimensional classification objective has the form 
	\begin{align*}
		\loss(w) = \frac{1}{N} \sum_{i=1}^{N} \ell(-y_i x_i\transpose w),
	\end{align*}
	where $y_i=1$ for all $i$, and there are $\m1$ copies of $e_1$, $\n1$ 
	copies of $-e_1$, $\m2$ copies of $e_2$, and $\n2$ copies of $-e_2$, with 
	$e_i$ being the $i$th coordinate vector in $2d$. Suppose $\ell$ 
	satisfies \cref{assmpt:individual-loss},
	and $\w1^* > \w2^*>0$. Then for all $\gamma\in(0, 1]$, 
	there exists a combination of $\m1$, $\n1$, $\m2$, and $\n2$ such that
	with a step size of $\eta = \gamma / \lambda_{\max}(\nabla^2 \loss(w^*))$,
	there exists a point $w=(\w1,\, \w2)\transpose$ with 
	\begin{align}
		\label{eq:before-crossing}
		\w1 > \w1^*, \quad \w2 > \w2^*, \quad \frac{\w2^*}{\w1^*} > \frac{\w2}{\w1}
	\end{align}
	and 
	\begin{align}
		\label{eq:after-crossing}
		\frac{\w2^*}{\w1^*} < \frac{T(w)_2}{T(w)_1}, \quad \text{and} \quad T(w)_1 > \w1^*.
	\end{align}
\end{lem}
\begin{proof}
	Using our dataset definition, the objective (up to scaling factor) is given by 
	\begin{align*}
		\loss(w) = \m1 \ell(-\w1) + \n1 \ell(\w1) + \m2 \ell(-\w2) + \n2 \ell(\w2).
	\end{align*}
	The gradient and Hessian are 
	\begin{align*}
		\nabla \loss(w) &= \begin{pmatrix}
			\n1\ell'(\w1) - \m1\ell'(-\w1)  \\
			\n2\ell'(\w2) - \m2\ell'(-\w2) 
		\end{pmatrix}\\
		\nabla^2 \loss(w) &= \begin{pmatrix}
			\m1 \ell''(-\w1) + \n1\ell''(\w1) & 0 \\
			0 & \m2 \ell''(-\w2) + \n2\ell''(\w2)
		\end{pmatrix}.
	\end{align*}
	Define $f_1(\w1) = \m1 \ell(-\w1) + \n1 \ell(\w1)$ and similarly for $f_2$, 
	so that $\loss(w) = f_1(\w1) + f_2(\w2)$. Due to this decoupling, running GD
	on $\loss$ is equivalent to running GD independently on the two 
	one-dimensional problems $f_1$ and $f_2$, with the same step size 
	\begin{align*}
		\eta = \frac{\gamma}{\lambda_{\max}(\nabla^2\loss(w^*))} = \frac{\gamma}{ \max\braces{ f_1''(\w1^*), \, f_2''(\w2*) } }.
	\end{align*}
	Note that we are free to set $\m1$, $\n1$, $\m2$, and $\n2$ such that 
	$\w1^* > \w2^*>0$, 
	\begin{align}
		\label{eq:2d-dataset-construction-conditions}
		f_1''(\w1^*) > f_2''(\w2^*) \quad \text{and} \quad \lim_{w\to\infty} f_1'(w) > f_2'(w)
	\end{align}
	so we will assume these are true without loss of generality. 

	Consider a point $w = (\w1,\w2)\transpose$ such that 
	the conditions in \labelcref{eq:before-crossing} are satisfied. Then 
	\begin{align*}
		T(w)_1 &= \w1 - \eta f_1'(\w1) \\
		&= \w1 - \frac{\gamma}{f_1''(\w1^*)} f_1'(\w1) \\
		&\geq \w1 - \frac{f_1'(\w1)}{f_1''(\w1^*)} \tag{$\gamma \leq 1$}\\
		&> \w1 - \frac{1}{f_1''(\w1^*)}( f_1''(\w1^*)(\w1 - \w1^*)) \tag{$f_1'$ is concave on $(\w1^*, \infty)$} \\
		&= \w1^* 
	\end{align*}
	where concavity of $f_1'$ follows from \cref{assmpt:individual-loss}.
	Furthermore, due to \labelcref{eq:2d-dataset-construction-conditions}, 
	there must exists $\bar w$ such that $f_1'(w) > f_2'(w)$ for all 
	$w > \bar w$. Thus we can pick $\w2 > \bar w$, and get 
	\begin{align*}
		T(w)_2 &= \w2 - \eta f_2'(\w2) > \w2 - \eta f_1'(\w2).
	\end{align*}
	Now let $\w2 = c\w1$ where $c \in(0,1]$. When $c=1$, we have 
	$T(w)_1 > c\w1 - \eta f_1'(c\w1) = T(w)_1$, and since $T(\w2)$ is 
	continuous in $c$, there must exist $c < 1$ where $T(w)_2 > T(w)_1$ 
	still hold. Thus we have shown that 
	\begin{align*}
		\frac{T(w)_2}{T(w)_1} > 1 > \frac{\w2^*}{\w1^*}.
	\end{align*}
	This completes the proof.
\end{proof}

\begin{lem}[Trajectory extension in 2D]
	\label{lemma:extending-gd-trajectory-2d}
	Consider the objective constructed in \cref{lemma:cross-over-2d}.
	For all $w'=(\w1',\w2')\transpose$ with 
	$\w1'>\w1^*$, $\w2'>\w2^*$, and $\eta>0$, there exists $w = (\w1,\w2)\transpose$
	such that $\w1 > \w1'$, $\w2>\w2'$, and $T(w)=w'$.
\end{lem}

\begin{proof}
	As shown in \cref{lemma:cross-over-2d}, the objective can be decoupled 
	into two independent problems corresponding to the two coordinates. 
	Therefore, we can simply extend the trajectory independently 
	in each coordinate using 
	\cref{lemma:extending-gd-trajectory}, and the result trivially holds.
\end{proof}

\begin{lem}[Continuous differentiability of perturbed loss]
	\label{lemma:ell-eps-cont-diff-2d}
	Let $\rho:\R^2\times \R\to\R$ be defined as 
	\begin{align*}
		\rho(\epsilon, w) = \begin{cases}
			\epsilon^2 \ell\paren{-\frac{w\transpose p}{\epsilon^2}}  & \text{if $\epsilon \neq 0$} \\
			\max\{-w\transpose p,\, 0\} & \text{if $\epsilon = 0$.}
		\end{cases}
	\end{align*}
	where $\ell$ satisfies \cref{assmpt:individual-loss},
	and $p$ is some non-zero vector in $\R^2$. 
	Then $\rho$ has continuous first partial derivatives 
	everywhere except on the set 
	\begin{align}
		\label{eq:non-differentiable-set-origin}
		\braces{(\epsilon, w):\; \epsilon=0 \; \text{and} \; w\transpose p = 0 }.
	\end{align}
	The second partial derivatives 
	$\nabla^2_w \rho$, $J_\epsilon \nabla_w\rho$, 
	and  $\nabla_w \frac{\partial \rho}{\partial \epsilon}$
	are continuous everywhere except on the same set 
	\labelcref{eq:non-differentiable-set-origin}, while 
	$\frac{\partial^2\rho}{\partial^2\epsilon}$ is only continuous on 
	$\braces{(\epsilon, w):\, w\transpose p > 0}$.
\end{lem}
\begin{proof}
The partial gradients and derivatives with respect to $w$ and $\epsilon$ 
are given by 
\begin{align}
	\nabla_w \rho(\epsilon, w) &= \begin{cases}
		-\ell'\paren{-\frac{w\transpose p}{\epsilon^2}} p & \text{if $\epsilon \neq 0$} \\
		-p & \text{if $\epsilon = 0$ and $w\transpose p<0$}\\
		0 & \text{if $\epsilon = 0$ and $w\transpose p>0$} 
\end{cases} \\
	\frac{\partial \rho}{ \partial \epsilon}(\epsilon, w) &= 2\epsilon\cdot\ell\paren{ - \frac{w\transpose p}{\epsilon^2}} + \frac{2w\transpose p}{\epsilon}\cdot\ell'\paren{ -\frac{w\transpose p}{\epsilon^2} },
\end{align}
when $\epsilon\neq 0$, and 
\begin{align}
	\left. \frac{\partial \rho }{ \partial \epsilon} (\epsilon, w) \right\rvert_{\epsilon=0} &= \lim_{t\to 0} \frac{ 1}{t} \paren{ t^2\ell\paren{-\frac{w\transpose p}{t^2}} - \max\braces{-w\transpose p,\, 0}  } = 0
\end{align}
as long as $w\transpose p \neq 0$. Continuity thus follows the same way as 
in the proof of \cref{lemma:ell-eps-cont-diff}, 
with $w$ replaced by $w\transpose p$ in appropriate places. 

The second derivatives involve taking the Jacobian of $\nabla_w\rho(\epsilon,w)$ 
with respect to $\epsilon$ and $w$ (the Hessian in this case), as well as 
computing $\frac{\partial^2 \rho}{\partial^2\epsilon}$ and 
$\nabla_w \frac{\partial \rho}{ \partial \epsilon}$. All arguments 
about continuity are once again the same as in \cref{lemma:ell-eps-cont-diff}, 
which lets us conclude that 
\begin{align*}
	\nabla^2_w \rho , \quad J_\epsilon \nabla_w\rho, \quad  \text{and} \quad \nabla_w \frac{\partial \rho}{\partial \epsilon}
\end{align*}
are all continuous everywhere except on the set 
\labelcref{eq:non-differentiable-set-origin},
where $J_\epsilon$ denotes the Jacobian with respect to $\epsilon$.
Lastly, $\frac{\partial^2\rho}{\partial^2\epsilon}$ is only continuous on 
$\braces{(\epsilon, w):\, w\transpose p > 0}$ 
(as in \cref{lemma:ell-eps-cont-diff}).
\end{proof}

\begin{lem}[Continuous differentiability of perturbed minimizer]
	\label{lemma:continuous-perturbation-wstar-epsilon-2d}
	Suppose \cref{assmpt:individual-loss} holds.
	As a function of $\epsilon$, the minimizer $w^*_\epsilon$ 
	of the perturbed objective \labelcref{eq:perturbed-loss-2d}
	is continuous everywhere on $\R$ 
	and continuously differentiable everywhere except at $\epsilon=0$. 
\end{lem}
\begin{proof}
	Continuity of $w^*_\epsilon$ follows the same argument as in 
	the one-dimensional case 
	(by \cref{lemma:continuous-perturbation-wstar-epsilon}), where 
	the horizon function $\loss^\infty_0(w) > 0$ for all $w$ not 
	at the origin follows again from convexity.
	Differentiability needs a bit more attention. As before, let 
	\begin{align*}
		G(\epsilon, w) = \nabla\loss_\epsilon(w).
	\end{align*}
	At point $(0, w^*)$, 
	$G(0, w^*) = \nabla\loss_0(w^*) = \nabla \loss(w^*) =0$,
	since the additional $\ell_{1,\epsilon}$ and $\ell_{2,\epsilon}$ 
	do not alter the solution at $\epsilon=0$ due to the $p$ and 
	$q$ points we have chosen. 
	Recall that 
	$p\transpose w^* = (\w1^*)^2 - (\w2^*)^2 > 0$ under our dataset construction 
	in \cref{lemma:cross-over-2d}. 
	Therefore, in a neighborhood of $\epsilon=0$ and $w=w^*$, 
	$G$ is continuously differentiable due to 
	\cref{lemma:ell-eps-cont-diff} and 
	twice continuous-differentiability of $\loss$. 

	The implicit function theorem (\cref{thm:implicit-function-theorem}) thus
	guarantees the existence of a continuously differentiable function 
	$\omega: I\to J$ where $I$ is an open interval about $0$ and $J$ an 
	open interval about $w^*$ such that 
	$\omega(0) = w^*$ and $\nabla\loss_\epsilon(\omega(\epsilon)) = 0$ 
	for all $\epsilon\in I$.
	As $\loss_\epsilon$ is strictly convex, $\omega(\epsilon)$ must be its unique 
	minimizer. Setting $w^*_\epsilon$ to $\omega(\cdot)$ completes the proof.
\end{proof}

\begin{lem}[Continuous differentiability of step size]
	\label{lemma:perturbed-stepsize-2d}
	Suppose \cref{assmpt:individual-loss} holds. As a function of $\epsilon$, 
	the step size 
	\begin{align}
		\eta_\epsilon = \frac{\gamma}{\lambda_{\max}( \nabla^2\loss_\epsilon(w^*_\epsilon) ) }
	\end{align}
	is continuously differentiable everywhere in $\epsilon$,
	where $\loss_\epsilon$ is defined as in 
	\labelcref{eq:perturbed-loss-2d}.
\end{lem}
\begin{proof}
	Recall that the original loss function we constructed in 
	\cref{lemma:cross-over-2d} has the form 
	\begin{align*}
		\loss(w) = \m1 \ell(-\w1) + \n1 \ell(\w1) + \m2 \ell(-\w2) + \n2 \ell(\w2),.
	\end{align*}
	The perturbations are given by 
	\begin{align*}
		\ell_{1,\epsilon}(w) = \begin{cases}
			\epsilon^2 \ell\paren{ -\frac{w\transpose p }{\epsilon^2} } & \text{if $\epsilon \neq 0$}  \\
			\max\braces{ -w\transpose p, \, 0} & \text{if $\epsilon=0$}
		\end{cases}  \qquad 
		\ell_{2,\epsilon}(w) = \begin{cases}
			\epsilon^2 \ell\paren{ -\frac{w\transpose q }{\epsilon^2} } & \text{if $\epsilon \neq 0$}  \\
			\max\braces{ -w\transpose q, \, 0} & \text{if $\epsilon=0$}
		\end{cases} 
	\end{align*}
	where $p = (\w1^*,\, -\w2^*)\transpose$, and $q = (0, \, \w2^*)\transpose$.
	The perturbed loss \labelcref{eq:perturbed-loss-2d} is 
	\begin{align*}
		\loss_\epsilon(w) \coloneqq \loss(w) + c_1 \ell_{1,\epsilon}(w) + c_2 \ell_{2, \epsilon}(w)
	\end{align*}
	for $c_1, c_2>0$. 
	Observe that the Hessian of the perturbed loss at $\epsilon\neq 0$ is simply 
	\begin{align*}
		\nabla^2\loss_\epsilon(w) &= \begin{pmatrix}
			\m1 \ell''(-\w1) + \n1\ell''(\w1) & 0 \\
			0 & \m2 \ell''(-\w2) + \n2\ell''(\w2)
		\end{pmatrix} \\
		&\qquad + c_1 \frac{1}{\epsilon^2}\ell''\paren{ -\frac{w\transpose p}{\epsilon^2} } pp\transpose + c_2 \frac{1}{\epsilon^2}\ell''\paren{ -\frac{w\transpose q}{\epsilon^2} } qq\transpose.
	\end{align*}
	
	We are interested in the eigenvalues of this Hessian evaluated at 
	$w^*_\epsilon$. 
	As $\nabla^2\loss_\epsilon(w^*_\epsilon)$ is just 
	a positive diagonal matrix plus two symmetric rank-$1$ updates, 
	its eigenvalues are also positive. As it's just a $2\times 2$
	matrix, the two (positive) eigenvalues can be found using the characteristic 
	polynomial. More importantly, they have closed-form expressions via the 
	quadratic formula. Combining with the fact that $\ell$ is three-times
	continuously-differentiable, each of the eigenvalues are also
	continuously differentiable in $\epsilon$ as long as $\epsilon\neq 0$.

	When $\epsilon=0$, $\nabla^2\loss_0(w^*_0) = \nabla^2\loss(w^*)$
	% as $w^*_\epsilon \to w^*$ as $\epsilon\to 0$, 
	and the additional 
	ReLU functions in the perturbation do not contribute to the 
	curvature. As $\epsilon\to 0$, by our \cref{assmpt:individual-loss} 
	on the decay rate of $\ell''$, the Hessian of $\ell_{1,\epsilon}$ 
	and $\ell_{2,\epsilon}$ will both go to zero as well, leaving 
	only that off the original loss $\loss$. 
	Thus we conclude that the eigenvalues of 
	$\nabla^2\loss_\epsilon(w^*_\epsilon)$ is continuous everywhere 
	in $\epsilon$. 

	Finally, as discussed in the proof of \cref{lemma:cross-over-2d}, we are 
	free to choose $\m1$ and $\n1$, which means we can scale them 
	simultaneously without changing $\w1^*$, to guarantee that 
	the first diagonal entry of $\nabla^2\loss(w^*_\epsilon)$ is as large 
	as we need, so that the largest eigenvalue of 
	$\lambda_{\max}( \nabla^2\loss_\epsilon(w^*_\epsilon) )$ 
	always occurs at the first coordinate. The proof is now complete.
\end{proof}

\begin{lem}[Continuous differentiability of perturbed GD map]
	\label{lemma:perturbed-T-differentiable-2d}
	The function $\phi:\R\times\R^2\to\R$
	\begin{align*}
		\phi(\epsilon, w) = T_\epsilon(w) = w-\eta_\epsilon\nabla\loss_\epsilon(w)
	\end{align*}
	is continuously differentiable everywhere except at $\epsilon=0$
	and $\braces{w:\, w\transpose p = 0}$,
	where $T_\epsilon$ is the perturbed GD map 
	in \labelcref{eq:perturbed-gd-map-2d}, and $p = (w_2^*, -w_1^*)\transpose$.
\end{lem}
\begin{proof}
	For $\epsilon\neq 0$, continuous differentiability of $\phi$ with respect to 
	$w$ follows from \cref{lemma:ell-eps-cont-diff-2d} and that of the original $\loss$.
	Additionally, recall that 
	\begin{align*}
		\eta_\epsilon = \frac{\gamma}{ \lambda_{\max}(\nabla^2\loss_\epsilon(w^*)) }.
	\end{align*}
	By \cref{lemma:perturbed-stepsize-2d}, this step size is continuously differentiable 
	in $\epsilon$. At $\epsilon=0$, the result trivially holds in $w$ by definition 
	of $\loss_\epsilon$.
\end{proof}
\section{Miscellaneous results}
\label{sec:misc}

\subsection{Properties of the logistic loss and the squareplus loss}
\label{sec:assumptions-on-ell-verification}

In this section, we verify that the logistic loss and the 
squareplus loss \citep{barron2021squareplus} both 
satisfy \cref{assmpt:individual-loss}.

\lossassumptions*

The logistic loss $\ell_{\text{logistic}}:\R\to\R$ is given by 
\begin{align}
	\ell_{\text{logistic}}(z) = \log(1+\exp(z)).
\end{align}
Its first two derivatives are 
\begin{align*}
	\ell'(z) = \sigma(z) \quad \text{and} \quad \ell''(z) = \sigma(z)(1-\sigma(z))
\end{align*}
where $\sigma(z)= 1/(1+\exp(-z))$ is the sigmoid function which 
is increasing from $0$ to $1$. 
Clearly, $\ell''$ is continuous and strictly positive for all $z\in\R$, 
so strict convexity follows. It is also monotonic on either side of $0$, 
which is where it is maximized. Lastly, 
\begin{align*}
	\lim_{\epsilon\to 0} \frac{1}{\epsilon^2} \ell''\paren{\frac{1}{\epsilon}} &= \frac{1}{\epsilon^2} \sigma'\paren{ \frac{1}{\epsilon} } \\
	&= \lim_{\epsilon\to 0} \frac{1}{\epsilon^2} \frac{ \exp(-1/\epsilon) }{ (\exp(-1/\epsilon) + 1)^2 }  \\
	&= \lim_{\epsilon\to 0} \frac{1}{2\epsilon^2} \exp(-1/\epsilon) \\
	&= \lim_{\epsilon\to 0} \frac{1}{2} \paren{ \frac{ \epsilon^2 }{ \exp(1/\epsilon) } }^{-1} \\
	&= 0
\end{align*}
as an exponential grows faster than a polynomial.

The squareplus loss $\ell_{\text{squareplus}}:\R\to\R$ 
was introduced by \citet{barron2021squareplus}. It is defined as 
\begin{align}
	\ell_{\text{squareplus}}(z) = 0.5(\sqrt{4+z^2}+z).
\end{align}
The first two derivatives are 
\begin{align*}
	\ell'(z) = \frac{1}{2} \paren{ \frac{z}{\sqrt{z^2+4}} +1 } \quad \text{and} \quad \ell''(z) = \frac{2}{ \paren{ z^2 + 4 }^{3/2} }.
\end{align*}
As the squareplus function was introduced to approximate the logistic loss 
(also known as the softplus activation), they share many similar properties.
The fact that 
\begin{align*}
	\lim_{\epsilon\to 0} \frac{1}{\epsilon^2} \ell''\paren{\frac{1}{\epsilon}} &= \lim_{\epsilon\to 0} \frac{1}{\epsilon^2} \paren{ \frac{2}{ \paren{ \frac{z}{\epsilon^2} + 4 }^{3/2} } } = 0.
\end{align*}

We now prove \cref{cor:limit-is-relu}.
\limitisrelu*
\begin{proof}
	At $z\leq 0$, the limit is clearly $0$ using the limit of $\ell$ as $z\to-\infty$.
	For all $z>0$, 
	\begin{align*}
		\lim_{\epsilon\to 0} \epsilon^2\ell\paren{\frac{z}{\epsilon^2}} &= \lim_{\epsilon\to 0} \frac{ \ell\paren{\frac{z}{\epsilon^2}}  }{ \epsilon^{-2} } \\
		&= \lim_{\epsilon\to 0} \frac{ -2\epsilon^{-3} z \ell'(z/\epsilon^2)  }{-2\epsilon^{-3}} \\
		&= \lim_{\epsilon\to 0} z \ell'(z/\epsilon^2) = z 
	\end{align*}
	using the fact that $\ell$ is convex with $\ell'$ upper bounded by $1$,
	which gives us the desired limit.
	
\end{proof}

\begin{fact}[Strict convexity of the logistic regression objective]
	\label{fact:logreg-strict-cvx}
	The logistic regression objective
	\begin{align*}
		\loss(w) = \sum_{i=1}^n \log(1+\exp(-y_iw\transpose x_i))
	\end{align*}
	with $x_i\in\R^d$ and $y_i=\braces{1,-1}$ is strictly convex on $\R^d$
	as long as the features $\braces{x_i}$ span $\R^d$.
\end{fact}
\begin{proof}
The objective's Hessian is given by 
	\begin{align*}
		\nabla^2\loss(w) = \sum_{i=1}^n \sigma'(y_iw\transpose x_i) x_ix_i\transpose 
		= X\transpose\diag(\sigma'(y\odot (Xw)))X
	\end{align*}
	where $y\in\R^n$ is the label vector and $\odot$ is the elementwise product.
	Since $\sigma'(z)>0$ for all $z$, the Hessian is positive definite 
	as long as $X\transpose X$ has full rank. 
\end{proof}

\subsection{Technical results in (convex) analysis}
\label{sec:cited-results}
Throughout our proofs, we extensively utilized the most basic form 
of the implicit function theorem \citep[Theorem 1.3.1]{krantz2002implicit}, 
stated as follows. 
\begin{thm}[Implicit function theorem]
	\label{thm:implicit-function-theorem}
	Let $G:\R^d \times \R\to\R$ be continuously differentiable 
	in an open neighborhood of $(x_0, y_0)$. Suppose that 
	\begin{enumerate}
		\item $G(x_0,y_0)=0$,
		\item $\frac{\partial G}{\partial y}(x_0,y_0)\neq 0$.
	\end{enumerate}
	Then there exist open set $I\subset\R^d$ and open interval $J$, 
	with $x_0\in I$, $y_0\in J$,
	and a unique, continuously differentiable function $\omega:I\to J$ satisfying 
	\begin{enumerate}
		\item $\omega(x_0) = y_0$,
		\item $G(x, \omega(x)) = 0$\, for all $x\in I$.
	\end{enumerate}
\end{thm}

The following statements can be found in \citet{rockafellar2009variational} as 
Definition 3.17, Theorem 3.21, and Corollary 7.43, respectively. They are 
used in showing continuity of $w^*_\epsilon$ in $\epsilon$ 
in our cycle construction (\cref{sec:1d-cycle}).

\begin{defn}[Horizon functions]
	\label{defn:horizon-functions}
	% Definition 3.17 in Rockafellar
	For any function $f:\R^n\to\bar \R$, the associated horizon function 
	is the function $f^\infty:\R^n\to\bar \R$ specified by 
	\begin{align}
		\epi f^\infty = (\epi f)^\infty \; \text{if} \; f\not\equiv \infty, \quad f^\infty = \delta_{\{0\}} \; \text{if} \; f\equiv \infty.
	\end{align}
\end{defn}

\begin{thm}[Properties of horizon functions]
	% Theorem 3.21 in Rockafellar
	\label{thm:horizon-functions}
	Suppose $f:\R^n\to\bar \R$ be convex, lsc, and proper.
	Then for any $\bar x\in\dom f$, 
	the horizon function $f^\infty$ is given by 
	\begin{align}
		f^\infty(w) = \lim_{\tau \to \infty} \frac{f(\bar x + \tau w) - f(\bar x) }{ \tau }
		= \sup_{\tau\in(0, \infty)} \frac{f(\bar x + \tau w) - f(\bar x)}{\tau}.
	\end{align}
\end{thm}

\begin{thm}[Solution mappings in convex optimization]
	% Corollary 7.43 in Rockafellar
	\label{cor:continuity-of-argmin}
	Suppose $P(u) = \argmin_x f(x,u)$ with $f:\R^n\times \R^m\to\bar R$ proper,
	lsc, convex, and such that $f^\infty(x,0) > 0$ for all $x\neq 0$. 
	If $f(x,u)$ is strictly convex in $x$, then $P$ is single-valued on $\dom P$
	and continuous on $\inte(\dom P)$.
\end{thm}
\section{Further details on the toy dataset}
\label{sec:toy-dataset-calculations}
\subsection{Gradient and Hessian}
Here we derive the gradient and Hessian for the toy dataset described in 
\cref{sec:bifurcation}. 
Recall that the dataset is constructed by setting $x_i=v$ 
for $i=1,\dots,n-1$ where $v$ is 
an arbitrary point on the $d$-dimensional unit sphere.
Then we let $x_n = -v$, and all labels are set to $y_i=1$. 
Clearly, this dataset is not separable by any linear classifier that goes
through the origin.
\begin{align*}
\loss(w) &= \frac{1}{n}\sum_{i=1}^n \log\paren{1+\exp(-y_i w\transpose x_i)} \\
\nabla \loss(w) &= \frac{1}{n}\sum_{i=1}^n \sigma(-y_i w\transpose x_i) (-x_i)
\end{align*}
where $\sigma(\cdot)$ is the sigmoid function. 
Plugging in the dataset, the gradient can be simplified as 
\begin{align*}
	\nabla \loss(w) &= \frac{1}{n}\paren{ (n-1)\sigma(-w\transpose v)(-v) + \sigma(w\transpose v)v } \\
	&= \frac{1}{n}\paren{ \sigma(w\transpose v) - (n-1)(1-\sigma(w\transpose v))  } v  \tag{Using $\sigma(w) + \sigma(-w) = 1$} \\
	&= \frac{1}{n}\paren{ n\cdot \sigma(w\transpose v) - (n-1) } v.
\end{align*}
Applying the chain rule gives us the Hessian
\begin{align*}
	\nabla^2 \loss(w) &= \frac{1}{n}  n\cdot \sigma'(w\transpose v) vv\transpose = \sigma'(w\transpose v) vv\transpose.
\end{align*}

\subsection{GD update in probability space}
\begin{prop}
	For any initialization $w_0\in\R^d$ and any time step $t$, GD iterations on $w$ generates 
	iterations of $p$ given by 
	\begin{align}
		p_{t+1,n} &= \sigma\paren{\sigma\inverse(p_{t,n}) - \frac{\eta}{n} \paren{ p_{t,n} - (n-1)(1-p_{t,n}) }} \nonumber \\
		p_{t+1,i} &= 1 - p_{t+1,n} \quad \forall i=1,\dots,n-1.
	\end{align}
\end{prop}
\begin{proof}
Recall the definition
$p_{t,n} = \sigma(-y_n w_t\transpose x_n)$. For all $i\neq n$ on the toy dataset,
\begin{align*}
	p_{t,i} &= \sigma(-y_i w_t\transpose x_i) \\
	&= \sigma(y_n w_t\transpose x_n) \tag{All labels are $1$ and $x_i=-x_1$}\\  
	&= 1-\sigma(-y_n w_t\transpose x_n) \tag{$\sigma(z) = 1-\sigma(-z)$} \\
	&= 1 - p_{t,n}
\end{align*} 
and so \labelcref{eq:one-d-toy-map} holds.
The update for $p_{t,n}$ is 
\begin{align*}
	p_{t+1,n} &= \sigma\paren{\sigma\inverse(p_{t,n}) - \frac{\eta}{n} y_n\paren{ \sum_{j=1}^n y_j p_{t,j}x_j\transpose }x_n } \\
	&= \sigma\paren{ \sigma\inverse(p_{t,n}) -\frac{\eta}{n} p_{t,n} + \frac{\eta}{n} \sum_{j=1}^{n-1} (1-p_{t,n}) } \\
	&= \sigma\paren{\sigma\inverse(p_{t,n}) - \frac{\eta}{n} \paren{ p_{t,n} - (n-1)(1-p_{t,n}) }  }
\end{align*}
where in the second step we used dataset construction 
along with \labelcref{eq:one-d-toy-map}. 

\end{proof}

\begin{prop}
	\label{prop:two-examples-oscillation}
	In our toy dataset, when $n=2$ and 
	$\eta\geq\eta_2=8$, the two points of oscillation are given by 
	\begin{align*}
		p_n = \frac{1}{2}\paren{ h\inverse\paren{\frac{\eta}{8} } + 1 },
	\end{align*}
	where $h(v) = \tanh\inverse(v)/v$, and $h\inverse$ is symmetric about the $x$-axis.
\end{prop}
\begin{proof}
	When $n=2$, the one-dimensional map for $p_n$ is given by 
	\begin{align*}
		\varphi(z) &= \sigma\paren{ \sigma\inverse(z) - \eta\paren{ z-\frac{1}{2} } }.
	\end{align*}
	Let $z_1$ and $z_2$ be 
	the two points in the period-$2$ orbit. That is, $\varphi(z_1) = z_2$
	and $ \varphi(z_2) = z_1$ while $z_1\neq z_2$ and neither are $0$ or $1$. 
	A $2$-cycle corresponds to a fixed point of the 
	$2$-iterate map $\varphi^{(2)} = \varphi(\varphi(z))$.
	Expanding it out, we find that 
	\begin{align}
		\label{eq:period-two-sum-to-1}
		\frac{1}{2} (z_1 + \varphi(z_1)) = \underbrace{\frac{1}{2} (z_1 + z_2)}_{\eqqcolon v} = \frac{n-1}{n}.
	\end{align}

	For a period-$2$ cycle, we have $z_2 = \varphi(z_1)$ and $z_1 = \varphi(z_2)$, 
	with $z_1 + z_2 = 1$ (see \labelcref{eq:period-two-sum-to-1}). Therefore, 
	a period-$2$ point $x=z_1$ or $x=z_2$ must satisfy 
	\begin{align*}
		1-x = \sigma\paren{ \sigma\inverse(x) - \eta \paren{ x - \frac{1}{2}} },
	\end{align*}
	which is equivalent to 
	\begin{align*}
		\ln\paren{ \frac{x}{1-x} } &= \frac{\eta}{2}\paren{ x - \frac{1}{2}}.
	\end{align*}
	Using a change of variable with $u\coloneqq 2x-1$, this is equivalent to 
	\begin{align*}
		\ln\paren{ \frac{1+u}{1-u} } = \frac{\eta}{4} u \quad \equiv  \quad \arctan(u) =  \frac{\eta}{8} u.
	\end{align*}
	Setting $h(v) \coloneqq \tanh\inverse(v)/v$ and substituting $x$ back 
	completes the proof.
\end{proof}

\begin{figure}
    \centering
    \includegraphics[width=0.8\linewidth]{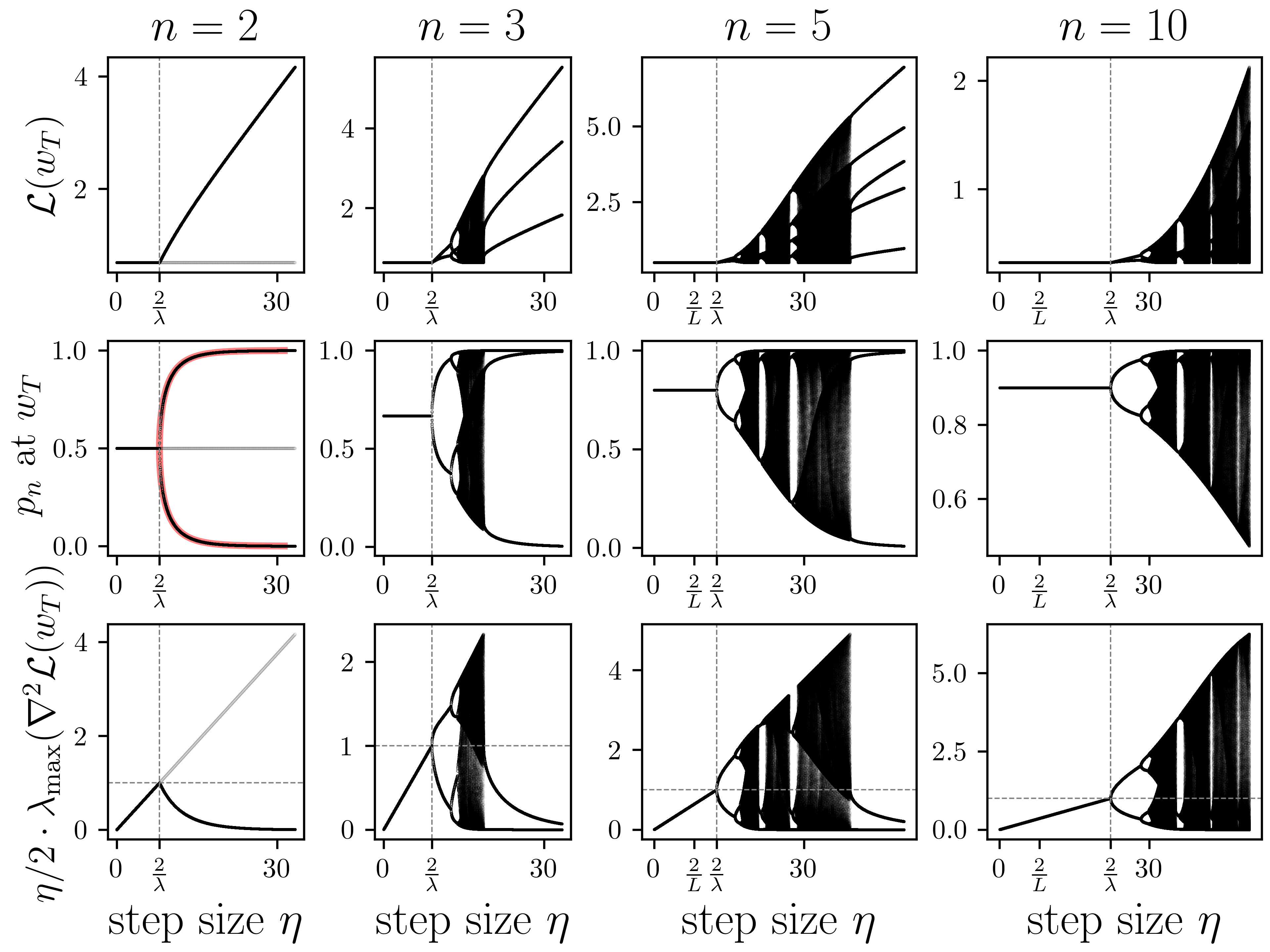}
    \caption{Bifurcation diagrams on the toy dataset for $d=2$ 
	and different values of $n$. For each step size $\eta$, 
	we run GD for $T=10^5$ iterations with 
	$1024$ random initializations of varying scales, 
    including one at $w_0=\mathbf{0}$. 
	Each point in a panel corresponds to the loss (first row), 
	the $p_n$ value in \labelcref{eq:one-d-toy-map}
    (second row), or the (scaled) largest eigenvalue of the Hessian evaluated 
	at $w_T$ (last row). When multiple points are visible for a single $\eta$, 
	GD either converged to a cycle or is chaotic under that step size. 
	The only exception to that is when $n=2$, where $w^*=\mathbf{0}$ is 
	the solution and thus a fixed point for any $\eta$. 
    For $n=2$, we have also superimposed in red the exact values 
	of $p_n$ computed from \labelcref{eq:two-examples-oscillation}.}
    \label{fig:toy_one_class}
\end{figure}

\section{Additional discussions}
\subsection{Basin of attraction of a two-dimensional counterexample}
\label{sec:initial-conditions}

Recall that when $\gamma<2$, any stable cycle that GD can converge to 
is in fact co-stable with the fixed point $w^*$. This means that 
depending on the initialization, GD may still converge to $w^*$ with the 
same step size, despite some other initialization may lead to 
convergence to the cycle. One may wonder, how big are their respective 
basins of attraction? That is, what are the set of initializations that 
eventually converge to the cycle (or $w^*$)? 
As our cycle construction critically relies on the implicit function theorem,
analyzing the size of the attracting neighbourhood is very challenging. 
Nonetheless, one can compute the basin of attraction numerically, 
as shown in \cref{fig:2d-cycle-basin} for a two-dimensional cycle. 
Not surprisingly, there is a neighbourhood around every point on the 
cyclic orbit that is attracted to the cycle. The shape and size of this basin, 
however, is rather fascinating, as it has a fractal-like structure 
and is much larger than just compact a neighborhood around the trajectory. 

\begin{figure}
	\centering \includegraphics[width=0.55\textwidth]{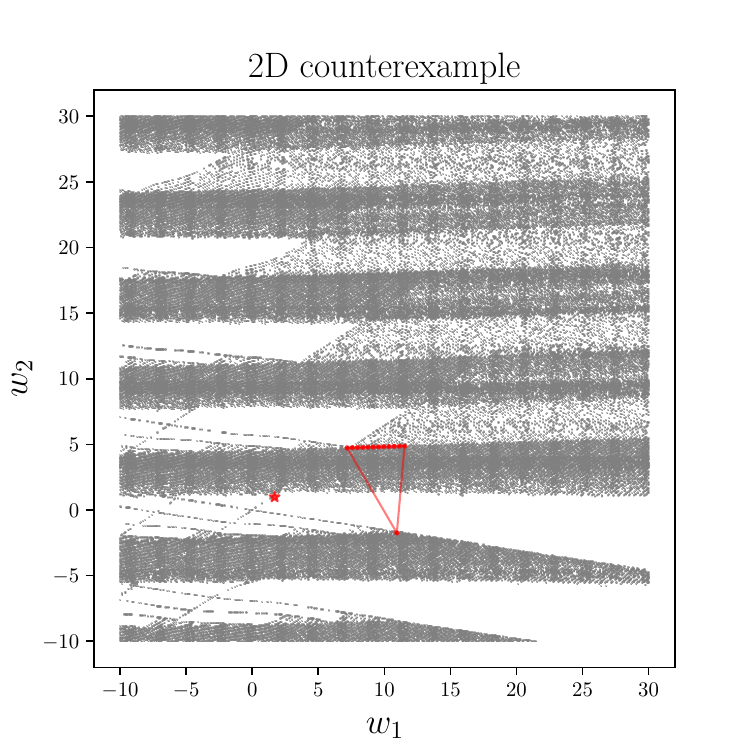}
	\caption{
		The basin of attraction of a two-dimensional counterexample. In this 
		two-dimensional logistic regression problem where the solution 
		is marked in a red star, GD can converge to a 
		cycle using $\gamma=0.95$ (red trajectory). We took a mesh grid of 
		size $512\times 512$ in the displayed range. We ran GD with 
		the prescribed step size at each point on this grid, and mark it 
		in grey if it converged to the cycle, and white if it converged to 
		$w^*$. The dataset construction is given in \cref{sec:exps}.
	}
	\label{fig:2d-cycle-basin}
\end{figure}

\subsection{Violation of EoS}
\label{sec:violation-of-eos}

In \cref{fig:cyclic-permutation}, we present a few logistic regression examples 
that violate the EoS phenomenon \citep{cohen2022gradient}
The construction of these examples are discussed as part of the proof for 
\cref{cor:lift-from-2d}, as well as 
at the end of \cref{sec:higher-dimension-case}. 

\begin{figure}
    \centering \includegraphics[width=0.9\textwidth]{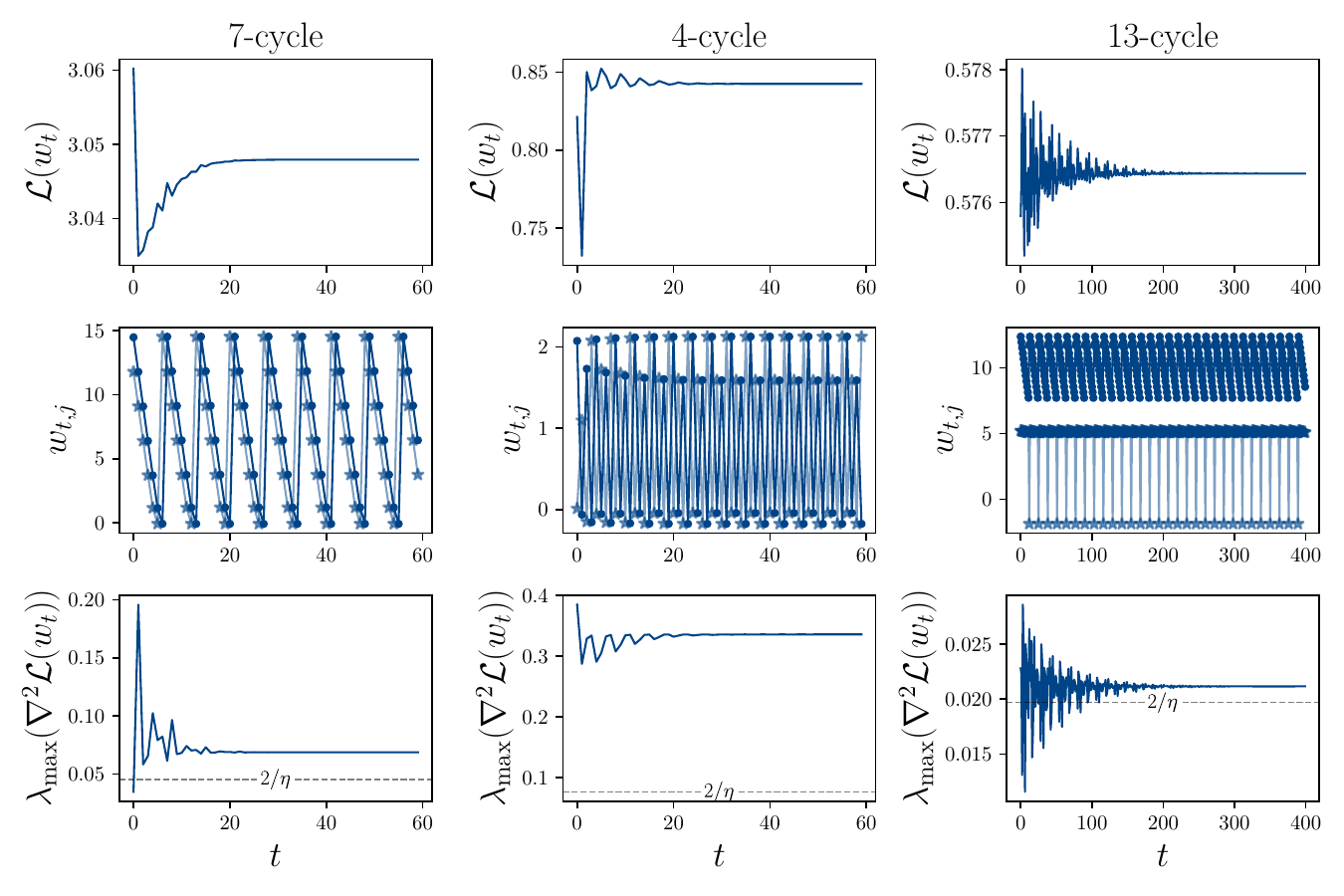}
    \caption{Logistic regression examples where GD using a step size strictly 
	below the stability threshold leads EoS violation. Observe that  
	while the sharpness oscillates around $2/\eta$ and eventually converges 
	to above $2/\eta$, the loss does not decrease further, and the 
	iterates are oscillating in a cycle. The first two examples 
	(first and second columns) are created using 
	the first two (one-dimensional) datasets in \cref{fig:cycles-w-diff-gammas}, 
	while the last example is created using the two-dimensional dataset 
	in \cref{fig:perturbed-loss-2d}. In the second row, two arbitrary coordinates 
	of the iterates are plotted.}
    \label{fig:cyclic-permutation}
\end{figure}

\subsection{Experiment details}
\label{sec:exps}

\paragraph{Dataset construction for \cref{fig:cycles-w-diff-gammas}}

We construct a one-dimensional dataset for each $\gamma$, where the labels are 
all $1$'s and the features are $x_i\in\braces{1,\,-1}$, with an extra $b$  
copies of an $x_i$ with a large magnitude. 
Let $m$ be the number of copies of $x_i=1$ and $n$ be the number of $x_i=-1$, 
with values given in \cref{tab:diff-gammas}
All initializations are set to $w_0=10$.

\begin{table}
	\caption{Dataset construction for \cref{fig:cycles-w-diff-gammas}}
	\centering
	\label{tab:diff-gammas}
	\begin{tabular}{@{}llllll@{}}
	\toprule
	$\gamma$ & $m$ & $n$ & Extra  $x_i$ & $b$ & GD converges to \\ \midrule
	$1.9$    & 250 & 200 & 20           & 6   & $4$-cycle       \\
	$1.5$    & 250 & 200 & 70           & 15  & $7$-cycle       \\
	$1.4$    & 200 & 190 & 270          & 25  & $37$-cycle      \\
	$1.5$    & 250 & 200 & 60           & 15  & Possibly chaos  \\ \bottomrule
	\end{tabular}
\end{table}

\paragraph{Dataset construction for \cref{fig:perturbed-loss-2d}}

The base dataset consists of $500$ copies of $x_i=e_1$, $30$ copies of 
$x_i=-e_1$, $5$ copies of $e_2$, and a single copy of $x_i=-e_2$. 
All labels are $1$'s.
On top of this, we add $7$ copies of $x_i = (45, \, -70)\transpose$
and $10$ copies of $x_i = (7.5, 50)\transpose$. 
These two sets of data points correspond to the two ``kicks'' required 
to form a cycle. The initialization is set to $w_0 = (15,\,4)\transpose$,
and $\gamma=0.4$.

\paragraph{Dataset construction for \cref{fig:2d-cycle-basin}}

This dataset is constructed nearly identical as the above, except that 
we only used $160$ copies of $x_i=e_1$, and $\gamma$ is set to $0.95$.

\end{document}